\documentclass[lettersize,journal]{IEEEtran}
\usepackage{amsmath,amsfonts}
\usepackage{algorithmic}
\usepackage{algorithm}
\usepackage{array}
\usepackage[caption=false,font=normalsize,labelfont=sf,textfont=sf]{subfig}
\usepackage{textcomp}
\usepackage{stfloats}
\usepackage{url}
\usepackage{verbatim}
\usepackage{graphicx}
\usepackage{cite}
\usepackage{amsthm} 
\newtheorem{assumption}{Assumption}
\newtheorem{definition}{Definition}
\newtheorem{theorem}{Theorem}
\newtheorem{lemma}{Lemma}
\usepackage{multirow}
\usepackage{booktabs}
\usepackage{xcolor}
\usepackage{enumitem}

\begin{document}

\title{Explanation-Guided Adversarial Training for Robust and Interpretable Models}


\author{Chao Chen, Yanhui Chen, Shanshan Lin, Dongsheng Hong, Shu Wu,~\IEEEmembership{Senior Member,~IEEE,} Xiangwen Liao, Chuanyi Liu
\thanks{Chao Chen and Chuanyi Liu are with Harbin Institute of Technology (Shenzhen). 
Yanhui Chen, Shanshan Lin, Dongsheng Hong, and Xiangwen Liao are with Fuzhou University.
Shu Wu is with NLPR, MAIS, Institute of Automation, Chinese Academy of Sciences.}
\thanks{Chao Chen and Chuanyi Liu were supported by the National Key Research and Development Program of China under Grant 2023YFB3106504.
Chao Chen, Yanhui Chen, Shanshan Lin, Dongsheng Hong, and Xiangwen Liao were supported by National Natural Science Foundation of China under Grant 62476060.
Shu Wu was supported by National Natural Science Foundation of China under Grant 62372454.}
}

\markboth{Journal of \LaTeX\ Class Files,~Vol.~14, No.~8, August~2021}%
{Shell \MakeLowercase{\textit{et al.}}: A Sample Article Using IEEEtran.cls for IEEE Journals}

\IEEEpubid{\footnotesize
Copyright \copyright\ 2026 IEEE. Personal use of this material is permitted.
However, permission to use this material for any other purposes must be obtained
from the IEEE by sending an email to pubs-permissions@ieee.org.}


\maketitle

\begin{abstract}
Deep neural networks (DNNs) have achieved remarkable performance in many tasks, yet they often behave as opaque black boxes. 
Explanation-guided learning (EGL) methods steer DNNs using human-provided explanations or supervision on model attributions. 
These approaches improve interpretability but typically assume benign inputs and incur heavy annotation costs. 
In contrast, both predictions and saliency maps of DNNs could dramatically alter facing imperceptible perturbations or unseen patterns. 
Adversarial training (AT) can substantially improve robustness, but it does not guarantee that model decisions rely on semantically meaningful features.
In response, we propose Explanation-Guided Adversarial Training (EGAT), a unified framework that integrates the strength of AT and EGL to simultaneously improve prediction performance, robustness, and explanation quality. 
EGAT generates adversarial examples on the fly while imposing explanation-based constraints on the model. 
By jointly optimizing classification performance, adversarial robustness, and attributional stability, EGAT is not only more resistant to unexpected cases, including adversarial attacks and out-of-distribution (OOD) scenarios, but also offer human-interpretable justifications for the decisions. 
We further formalize EGAT within the Probably Approximately Correct learning framework, demonstrating theoretically that it yields more stable predictions under unexpected situations compared to standard AT. 
\textcolor{black}{Empirical evaluations on OOD benchmark datasets show that EGAT consistently outperforms competitive baselines in both clean accuracy and adversarial accuracy (\(+37\%\)) while producing more semantically meaningful explanations, and requiring only a limited increase (\(+16\%\)) in training time. 
}
\end{abstract}


\section{Introduction}
Deep neural networks (DNNs) have demonstrated exceptional performance across a multitude of domains. 
However, in safety- or regulation-critical applications, 
stakeholders demand models that not only classify correctly but also provide understandable justification for their decisions \cite{guidedlearning}. 
For instance, a medical AI system should highlight diagnostic features of an image \cite{liu2024xfmp}, and financial regulators require transparent decision reasoning \cite{weber2024applications}.  
\textit{Explanation-guided learning} (EGL) has emerged to meet this demand by incorporating human knowledge directly into training \cite{ismail2021sgt,rieger2020cdep,selvaraju2019hint}.
EGL methods impose additional loss terms or regularizers on model explanations (e.g., saliency maps), encouraging the network to focus on relevant regions.  
For example, some approaches use coarse annotations such as bounding boxes \cite{rieger2020cdep} to constrain the model's attention (penalizing attributions outside known object regions), 
and others use more detailed rationales to align predictions with expert-provided explanations \cite{zhang2023magi,shen2021iai}.  
These methods improve interpretability by aligning the model's reasoning with human-understandable features. 

However, DNNs are known to be highly vulnerable to imperceptible input perturbations.
Even minor changes to the input, often indistinguishable to humans, can lead to incorrect predictions \cite{goodfellow2014fgsm,madry2017pgd}. 
Beyond accuracy degradation, such perturbations can also cause significant shifts in saliency maps, resulting in inconsistent and potentially misleading explanations \cite{ghorbani2019interpretation,sarkar2021enhanced}.
This highlights a critical limitation: 
EGL models trained \textit{solely} with explanation supervision remain susceptible to adversarial manipulation and may fail to generalize to unseen or out-of-distribution (OOD) patterns. 
Meanwhile, the robustness of EGL methods are underexplored: 
Only a few EGL-based approaches \cite{li2023dre} have \textit{empirically} examined robustness under OOD settings, 
yet these approaches offer no formal \textit{theoretical analysis} and do not explicitly address \textit{adversarial robustness}. 
Consequently, the reliability of their explanations and predictions under real-world perturbations remains uncertain.

\textit{Adversarial Training} (AT) \cite{madry2017pgd,goodfellow2014fgsm} is one of the leading defense strategies: 
it augments training with worst-case examples to significantly improve resistance to attacks. 
Several variants of AT have been proposed to stabilize model attributions under attack (e.g., by minimizing changes in the saliency map during training \cite{chen2019ignorm,chen2024training,wicker2022robust}). 
These attributional robustness methods ensure consistent predictions or explanations, but they do not explicitly ensure that the model is basing its predictions on those salient regions.
In other words, a network could maintain stable attributions under attacks while still \textit{relying on uninterpretable or spurious features}.
This limitation implies robust models may focus on irrelevant pixels or background cues, provided their explanation appears stable.
In short, there is \textit{a fundamental gap}: AT improves worst-case accuracy but neglects explanation semantics, while EGL aligns explanations but neglects model robustness.

In this work, we propose Explanation-Guided Adversarial Training (EGAT), a novel training framework that bridges the strengths of \textcolor{black}{AT} and \textcolor{black}{EGL}.
EGAT introduces a unified objective that not only aligns explanations under adversarial perturbations but also trains the model to base its predictions on semantically meaningful and human-interpretable features.
Concretely, we incorporate multiple saliency-driven constraints into the objective during adversarial training to \textit{directly} guide the model’s focus towards meaningful regions. 
The key components include: 
(1) an explanation adversarial loss to keep the model’s highlighted regions and predictions stable under attacks; 
(2) a consistency loss with priori interpretable regions to algin model's explanations with human-understandable areas and discourages attributions on irrelevant regions. 
By unifying these constraints with adversarial training, EGAT explicitly trains the model to rely on the salient features for its predictions, to enhance its interpretability, robustness and generalization. 

We justify the effectiveness of EGAT through both theoretical analysis and empirical evaluation.
\textit{Theoretically}, we formalize EGAT within the Probably Approximately Correct (PAC) learning framework and establish generalization guarantees under distributional shifts. 
Specifically, we show that EGAT yields models with provably lower generalization error in the presence of unseen perturbations compared to standard \textcolor{black}{AT}. 
\textit{Empirically}, we comprehensively evaluate EGAT's performance under both OOD settings and adversarial attacks. 
We conduct experiments on eight datasets derived from two widely used benchmarks that exhibit significant domain shifts and spurious correlations. 
EGAT is compared against \textcolor{black}{AT} methods such as DMADA \cite{xu2020dmada}, IGR \cite{ross2018improving} and IGN \cite{chen2019ignorm}, as well as leading \textcolor{black}{EGL} baselines including IRM \cite{arjovsky2019irm}, DRE \cite{li2023dre}, and SGDrop \cite{bertoin2024sgdrop}.
EGAT achieves superior performance across multiple criteria, especially high clean accuracy on both in-distribution data and out-of-distribution data, along with superior adversarial robustness under attack. 
Furthermore, we quantitatively and qualitatively evaluate explanation quality by standard metrics, including comprehensiveness and sufficiency, and case studies, respectively. 
Results show that EGAT produces explanations that are not only stable and human-aligned but also better reflect the true decision-making process of the model.

In sum, this work makes the following key contributions:
\begin{itemize}
    \item \textcolor{black}{We propose a novel training framework, namely EGAT,} that explicitly enforce predictions to stably depend on semantically meaningful regions.
    EGAT integrates multiple explanation-centric losses, including adversarial saliency consistency, region alignment with interpretable priors, and mixup attribution matching. 

    \item We provide a principled \textit{theoretical} analysis of EGAT within the 
    \textcolor{black}{PAC}
    learning framework, establishing generalization bounds on its adversarial risk. Besides, out-of-distribution stability of EGAT and traditional AT are carefully analyzed.
    
    \item We provide extensive \textit{empirical} evaluation on two domain generalization benchmarks, demonstrating that EGAT achieves superior clean and adversarial accuracy, better generalizability, and explanation faithfulness compared to a wide range of state-of-the-art baselines.
    
\end{itemize}

\textcolor{black}{
\noindent\textbf{Positioning with TCSVT literature.}
Our work is closely related to several recent papers published in \emph{IEEE Transactions on Circuits and Systems for Video Technology (TCSVT)} that study adversarial robustness and explanation-related mechanisms in vision/video tasks, including \cite{Chen2023RobustNIC,Fu2024AttackInvariance,Sun2024TargetedAttacksOD,Cheng2025ExplainAttack3D}. 
Compared to these works, EGAT is distinctive in that it (1) is a \emph{training-time defense} that explicitly optimizes both \emph{prediction robustness} and \emph{explanation faithfulness} under \textit{universal} perturbations, rather than improving robustness alone \cite{Fu2024AttackInvariance,Chen2023RobustNIC} or attacking successful rates \cite{Cheng2025ExplainAttack3D,Sun2024TargetedAttacksOD,Chen2023RobustNIC}; (2) uses explanations as \emph{supervision signals} to encourage models to rely on semantically meaningful regions, instead of leveraging explainability mainly for attacking models \cite{Cheng2025ExplainAttack3D,Sun2024TargetedAttacksOD}; and (3) provides a PAC-style theoretical analysis that incorporates explanation-related terms, and is validated empirically using both robustness metrics and explanation-quality metrics for adversarial attacks and OOD scenarios.
}

\section{Related Work}

\subsection{Explanation-guided learning}
\textcolor{black}{EGL}
methods incorporate human-aligned explanations \cite{zhao2022toward,li2024towards,ross2017right} into the training process to steer models towards more interpretable reasoning. 
The ``Right for the Right Reasons'' framework \cite{ross2017right} introduces penalizes gradients on features deemed irrelevant (based on expert annotations), and encourages the model to make predictions using the ``right'' evidence. 
Many subsequent works extend this idea of guiding model training with explanations.
Specifically, by minimizing the discrepancy between the model's saliency map and ground-truth regions of interest, HINT \cite{selvaraju2019hint} and CDEP \cite{rieger2020cdep} could bias the network to attend to the correct human-understandable visual cues. 
However, one potential drawback is that the cues are typically derived from domain experts. 

More recently, SGT \cite{ismail2021sgt} masks out low-saliency regions of an input and constrains the model’s predictions to remain unchanged for the masked versus original input. It adopts KL-divergence to align the output distributions of original and masked examples.
SGDrop \cite{bertoin2024sgdrop} selectively drops the \textit{most} salient features during training to combat the model's tendency to over-rely on a small set of predictive pixels. By occluding highly activated features (according to the model's own saliency), SGDrop encourages the network to consider a broader set of cues and reduces overfitting. 
DRE \cite{li2023dre} enforces consistency of a model’s attribution maps \textit{across different data distributions}. 
Without requiring human annotation, DRE treats shifts between training subsets as self-supervision signals, encouraging the model’s explanations to focus on invariant, causally relevant features rather than dataset-specific quirks. 

Regarding robustness of EGL models, 
DRE \cite{li2023dre} shows good generalization for OOD scenarios empirically \textit{without theoretical analysis}.
As for adversarial attacks, 
ASGT \cite{guesmi2024asgt} explicitly accounts for adversarial manipulation of inputs, while they ignore the explanation stability under attacks \textit{empirically}. 
Consequently, models trained with only explanation supervision can still be vulnerable to adversarial attacks that exploit features outside the human-aligned regions. 
\textcolor{black}{
Conversely, explanations are also leveraged to strengthen attacks: 
Cheng et al.~\cite{Cheng2025ExplainAttack3D} use explanations to guide adversarial attacks in 3D tracking, and Sun et al.~\cite{Sun2024TargetedAttacksOD} study targeted attacks against object detectors with task-specific importance considerations. 
These works are primarily attack-oriented and do not aim to improve adversarial robustness.
}
This gap motivates our work that explicitly integrate robustness considerations into the explanation alignment process, and theoretically analyze model robustness.

\subsection{Robust Machine Learning}
\textcolor{black}{AT}
for \textit{prediction} robustness was first introduced by Goodfellow et al. as a practical defense using FGSM perturbations \cite{goodfellow2014fgsm}, 
and later refined by Madry et al. through a robust optimization framework using stronger \textcolor{black}{PGD} attacks \cite{madry2017pgd}.
These techniques yield models with substantially higher prediction accuracy on adversarial inputs, effectively providing a security guarantee against certain threat models \cite{madry2017pgd}. 
More recently, AT for various settings and specific attacks have been well studied, such as pairwise AT for domain adaptation \cite{shi2024npat} and detection for Deepfake \cite{li2022adal}. 
\textcolor{black}{Chen et al.~\cite{Chen2023RobustNIC} study robustness issues in neural image compression,
while Fu et al.~\cite{Fu2024AttackInvariance} improve adversarial generalization by encouraging attack-invariant behavior.}

Beyond prediction robustness, researchers also work on \textit{explanation robustness}. 
Ghorbani et al. \cite{ghorbani2019interpretation} showed that one can craft imperceptible perturbations to an image which leave the model's prediction unchanged but dramatically distort its saliency map. 
Similarly, Slack et al. \cite{slack2020fooling} demonstrated attacks on model-agnostic explainers,
where an adversary can fool the explainer into highlighting arbitrary features without affecting the true output. 
In response, some works aligned explanations across transformed inputs. For instance, DRE \cite{li2023dre} and AFT \cite{wu2025aft} enforce that a model's saliency for the original image and for its geometrically-transformed version remain consistent. 
However, none of them evaluate the adversarial robustness by experiments nor theorems.

Furthermore, a line of work has emerged on \textcolor{black}{AT} techniques aimed at stabilizing attribution maps. 
One simple yet influential idea is IGR \cite{ross2018improving}. 
They found that penalizing the magnitude of a model's gradient with respect to the input both improve adversarial robustness and produce more globally coherent attributions. 
Tsipras et al. \cite{tsipras2018robustness} found that \textcolor{black}{AT} could guide model to focus on shape-based features that align with human perception, while they do not explicitly consider robust accuracy as well as explanation robustness.
IG-NORM \cite{chen2019ignorm} directly incorporates an attribution consistency  to penalty the discrepancy between IG-based \cite{sundararajan2017ig} attribution maps for the original input and its counterpart. 

Despite improving the resilience of saliency maps, current adversarial explanation techniques strive to preserve the \textit{shape} or \textit{stability} of the explanation under perturbations, but do not guarantee that the model is truly using the ``right'' features. 
A model could have a very stable attribution map that is nonetheless pointing to a misleading region.

In sum, prior works either align explanations to predefined annotations (EGL methods) \textit{or} enforce explanation invariance under attack, but seldom both. 
Instead, our proposed EGAT explicitly incorporates an explanation alignment objective into the adversarial training loop.

\section{Preliminaries}
We begin by introducing three key foundational concepts: Explainable Machine Learning, \textcolor{black}{EGL, and AT}, which are necessary for understanding our proposed EGAT.

\subsection{Explainable Machine Learning and Grad-CAM}
Let $S_{train} = \{(x_i, y_i)\}_{i=1}^n$ denote a training dataset, where $x_i \in \mathbb{R}^d$ and $y_i \in \{1, \ldots, C\}$ are an input instance and its label, respectively. 
We consider a classifier $f_\theta: \mathbb{R}^d \rightarrow \mathbb{R}^C$ parameterized by $\theta$, which outputs probability distribution for each $x_i$. We will omit $\theta$ if no ambiguous, and use $f(x)_c$ denote the predicted probability for class $c$. 
The primary goal of \textcolor{black}{explainable machine learning} is to understand which input features influence the model's prediction. 
Formally, attribution methods \cite{selvaraju2020gradcam,sundararajan2017ig} define an explanation function $\Phi: \mathbb{R}^d \rightarrow \mathbb{R}^d$ that assigns an importance score to each input dimension, such that $\Phi(x,c)_j$ reflects how much the $j$-th feature contributed to $f_\theta(x)$ on class $c$.

Among many attribution techniques, Grad-CAM \cite{selvaraju2020gradcam} highlights a spatial heatmap and is widely used in computer vision. 
Formally, to explain the prediction for class $c$, Grad-CAM computes the channel-wise weights as
\begin{equation}
    \alpha^k_c = \frac{1}{Z} \sum_{i,j} \frac{\partial f(x)_c}{\partial A^k_{i,j}},
\end{equation}
where $A^k$ is the $k$-th activation map of a convolutional layer, $i,j$ specifies the location of the pixels in images or entries in the intermediate feature maps, and $Z$ is the number of spatial locations. 
Then, Grad-CAM saliency map is defined by: 
\begin{equation}
\label{eq:grad_cam}
\Phi(x,c) = \text{ReLU}\left( \sum_k \alpha^k_c A^k \right).    
\end{equation}

\subsection{Explanation guided learning}
\textcolor{black}{EGL} aims to improve the faithfulness and transparency of models by incorporating attribution-based signals into training objectives \cite{ross2017right,rieger2020cdep,ismail2021sgt,li2023dre,bertoin2024sgdrop}.
These EGL techniques generally rely on an auxiliary supervision signal, such as an explanation mask $M(x)$, 
and the mask is typically defined by human annotations \cite{ross2017right,rieger2020cdep}. 
To guide the model's attributions and attentions,
EGL enforces that $\Phi(x)$ aligns with $M(x)$ via a regularization term:
\begin{equation}
\label{eq:preliminary_egl}
    \mathcal{L}_{egl} = d\big(\Phi(x), M(x)\big),    
\end{equation}
where $d(\cdot,\cdot)$ is a divergence measure, 
such as $L_1$ and \textcolor{black}{Binary Cross Entropy (BCE)}.

\subsection{Adversarial Training}
\textcolor{black}{AT} \cite{goodfellow2014fgsm, madry2017pgd} is widely used to improve a model's resilience to adversarial attacks.
To achieve the goal, 
AT exposes models to ``worst-case'' perturbations, which are small but malicious changes to the input that drastically alter model behaviors. 
Formally, it solves the following min-max problem:
\begin{equation}
\label{eq:preliminary_at}
    \min_\theta \mathbb{E}_{(x, y) \in S_{train} } \left[ 
    \max_{\delta:\|\delta\| \leq \epsilon}
    \ell
    (f_\theta(x + \delta), y) \right],
\end{equation}
where the inner maximization generates adversarial examples using attacks like FGSM \cite{goodfellow2014fgsm} or PGD \cite{madry2017pgd}, and the outer minimization updates model parameters.

\section{Methodology}
In this section, we will introduce \textcolor{black}{EGAT}, which combines \textcolor{black}{AT with EGL} in a unified model training process. 
In Sec. \ref{sec:egat_framework}, we formalize the EGAT framework and detail its components. 
In Sec. \ref{sec:egat_theoretical_pac}, we theoretically analyze EGAT in the view of PAC learning. 
By comparing with standard AT, we justify EGAT's objective for improvement in prediction robustness under unexpected situations.

\begin{figure*}
    \centering
    \includegraphics[width=0.98\linewidth]{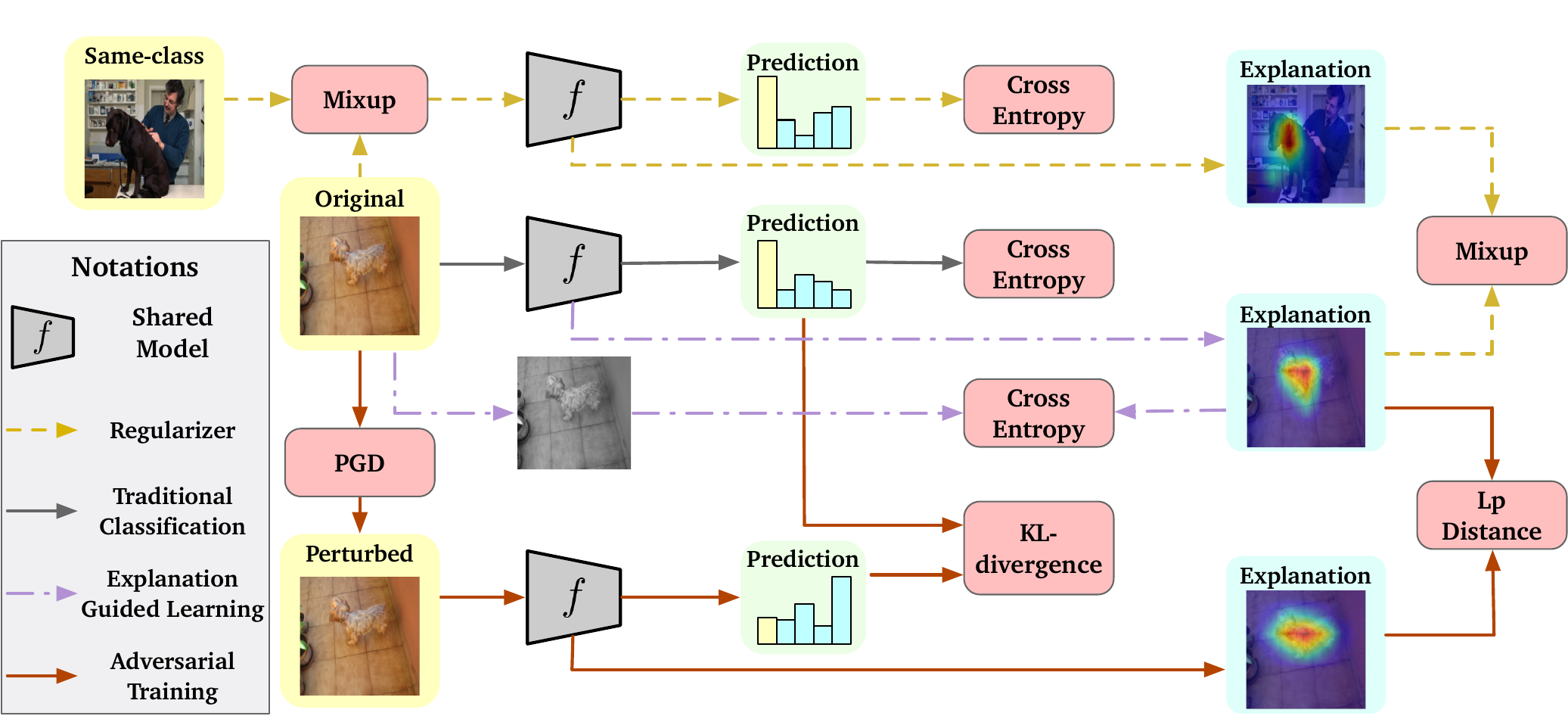}
    \caption{Overview of the proposed EGAT framework:
    For each original input, EGAT optimizes a complementary objectives of four terms:
    (1) Traditional classification loss ($\mathcal{L}_{cls}$, in gray) is simply calculated by cross entropy; 
    (2) Adversarial training based loss ($\mathcal{L}_{adv}$, in red) promotes the robustness for both predictions and explanations;
    (3) Explanation guided learning loss ($\mathcal{L}_{egl}$, in purple) adopts explanations as extra supervision signals;
    (4) Regularizer ($\mathcal{L}_{reg}$, in yellow) incorporates mixup strategy to further regulate model behaviors.
    }
    \label{fig:framework}
\end{figure*}

\subsection{Explanation-Guided Adversarial Training Framework}
\label{sec:egat_framework}
As shown in Fig. \ref{fig:framework}, 
the proposed \textcolor{black}{EGAT} framework integrates \textcolor{black}{AT} with saliency-based supervision to enhance both prediction robustness and explanation faithfulness. 

\subsubsection{Overall Objective}
EGAT is motivated by the need to make models not only robust to adversarial perturbations, but also interpretable in a consistent and trustworthy way. 
Most existing \textcolor{black}{AT} methods focus solely on label robustness, while \textcolor{black}{EGL} methods may ignore adversarial behavior. EGAT bridges this gap by considering three key objectives: 
(1) $\mathcal{L}_{adv}$ avoid model's predictions and explanations sensitive to small but potentially malicious perturbations;
(2) $\mathcal{L}_{egl}$ encourages that the model grounds its predictions on the salient and interpretable regions of the input;
(3) $\mathcal{L}_{reg}$ complements AT by regularizing model behavior within a larger region.
In sum, EGAT optimizes the following complementary objectives: 
\begin{equation}
\label{eq:obj_overall}
    \mathcal{L}_{EGAT} = \mathcal{L}_{cls} 
    + \lambda_1 \mathcal{L}_{adv} 
    + \lambda_2 \mathcal{L}_{egl}
    + \lambda_3 \mathcal{L}_{reg} ~,
\end{equation}
where $\mathcal{L}_{cls}$ is the standard classification loss, e.g., cross-entropy. 
Each $\lambda_i\geq 0, i\in\{1,2,3\}$ controls the relative importance of the associated term. 
In the following, we further discuss $\mathcal{L}_{adv}$ in Sec. \ref{sec:egat_at}, $\mathcal{L}_{egl}$ in Sec. \ref{sec:egat_egl}, and $\mathcal{L}_{reg}$ in Sec. \ref{sec:egat_reg}.
We empirically justify each component's contribution by ablation studies in Sec. \ref{sec:ablation_study}, where any of these losses are removed lead to drops in either predictive robustness or explanation quality.

\subsubsection{Adversarial Training}
\label{sec:egat_at}
To directly promote adversarial robustness, EGAT incorporates \textcolor{black}{AT} into its objective, ensuring that both the model's predictions and their corresponding explanations remain stable under input perturbations. 
Formally, for each input sample $x$, we use \textcolor{black}{PGD} \cite{madry2017pgd} to compute an adversarial example $x^{adv} = x + \delta^\ast$ by solving:
\begin{equation}
\label{eq:adv_sample}
    \delta^\ast = \arg \max_{\|\delta\| \leq \epsilon} \ell(f_\theta(x + \delta), y),
\end{equation}
where $\ell(\cdot)$ is the standard classification loss (e.g., cross-entropy), and $\epsilon$ controls the perturbation strength. 
\textcolor{black}{
We adopt PGD because it is a widely used and standard choice in AT studies, and it provides a strong first-order adversary that generates informative training-time perturbations.
}

Unlike traditional AT methods that focus solely on prediction consistency, EGAT enforces consistency in both predictions and explanations. 
Specifically, the adversarial loss in EGAT is formulated as:
\begin{equation}
\label{eq:egat_adv}
    \mathcal{L}_{adv} = KL \big(f(x)\, \| \, f(x^{adv}) \big) + 
    \lambda_4 \|\Phi(x), \Phi(x^{adv})\|_2,
\end{equation}
where the first term is the KL divergence between the output distributions on the clean input $x$ and its adversarial counterpart $x^{adv}$, encouraging prediction stability. 
The second term penalizes discrepancies in Grad-CAM saliency maps between $x$ and $x^{adv}$, thereby promoting explanation consistency. 

The underlying intuition is that if a small, imperceptible perturbation does not significantly alter a model's internal activation patterns (as reflected by Grad-CAM), 
then its output prediction should remain stable as well. 
The inclusion of the explanation consistency term is further motivated by Lemma~\ref{lemma:ood_explanation_stability}, which formally establishes that explanation stability plays a critical role in ensuring prediction robustness under distributional shifts and adversarial conditions.

\subsubsection{Explanation guided learning}
\label{sec:egat_egl}
To guide the model's attention toward perceptually meaningful regions, we introduce an explanation alignment loss that penalizes discrepancies between the model's attribution map $\Phi(x)$ and a reference importance mask $M(x)$. Formally:
\begin{equation}
\label{eq:egat_egl}
    \mathcal{L}_{egl}=BCE\big(\Phi(x),M(x)\big)~.
\end{equation}
EGAT is agnostic to the choice of gradient-based method \cite{sundararajan2017ig,selvaraju2020gradcam} and can be applied with these works seamlessly. 
We adopt Grad-CAM \cite{selvaraju2020gradcam} in Equation (\ref{eq:grad_cam}) as the default saliency generator 
for fair comparisons to existing works, since DRE \cite{li2023dre} 
utilize Grad-CAM as explanation methods.
Besides that, existing work \cite{adebayo2018sanitychecks} implies that only vanilla gradient and Grad-CAM pass sanity checks among gradient-based methods, 
where vanilla gradients could introduce instability and sensitivity to noise \cite{sundararajan2017ig}.

Prior works often rely on handcrafted or expert-annotated masks to define $M(x)$, such as segmentation maps of target objects or bounding-box regions \cite{ross2017right,rieger2020cdep,zhang2023magi,zhuang2019care}. 
However, acquiring such labels incurs considerable annotation costs and limits scalability. 
To reduce reliance on external supervision, methods like SGT \cite{ismail2021sgt} instead extract $M(x)$ by picking the pixels with the largest gradient magnitudes.
In contrast, EGAT employs a simple and effective strategy: 
we transform the original RGB input image $x$ into its grayscale counterpart, and treat it as a soft proxy for visual saliency. This grayscale map captures coarse contrast structure and naturally highlights object boundaries, and luminance patterns that often correlate with semantic content. 

Intuitively, minimizing $\mathcal{L}_{egl}$ leads to two key benefits:
(1) It makes the model's decisions more interpretable, since the highlighted regions align with human-understandable concepts; 
(2) It improves robustness, because the model is less likely to be distracted by confounding features \textit{outside} the region of interest.

\subsubsection{Regularizers}
\label{sec:egat_reg}
Beyond adversarial and explanation-guided objectives, EGAT incorporates additional regularizers for better generalization. 
To complement AT, which optimizes a tiny local region of inputs, 
mixup strategy \cite{guo2019mixup,xu2020dmada} reduces vulnerability to unforeseen data variations within a larger region.
Mixup generates new training samples by linearly interpolating between existing examples. 
Formally, 
given two training instances $(x_i, y_i)$ and $(x_j, y_j)$, where $y_i = y_j$, a synthetic input is created as:
\begin{equation}
    \tilde{x} = \beta x_i + (1 - \beta) x_j, \quad \textnormal{where } \beta \sim \text{Beta}(\alpha, \alpha).
\end{equation}

Again, EGAT extends the mixup principle to both prediction and explanation spaces: 
\begin{equation}
    \mathcal{L}_{reg} = CE \big(f(\tilde{x}), y \big) 
    + \| \tilde\Phi(x)-\Phi(\tilde{x}) \|_1 
    + \|\Phi(x)\|_1 ~,
\end{equation}
where $\tilde \Phi(x)=\lambda \Phi(x_i)+(1-\lambda)\Phi(x_j)$ is the blended attribution maps. 
The first two terms enforce that the model produce correct predictions with reasonable explanations, and the last term promotes concise attribution maps.

\subsection{PAC-Theoretic Analysis for EGAT}
\label{sec:egat_theoretical_pac}
In this subsection, we theoretically analyze EGAT in the view of 
\textcolor{black}{PAC learning},
adapted for adversarial robustness following the analyses in \cite{viallard2021pac,xiao2022adversarial}. 
Specifically, we introduce basic notations and definitions in Sec. \ref{sec:pac_notation}.
Then, we provide EGAT's generalization bounds in Sec. \ref{sec:generlization_bound_egat}, 
and explicitly compare EGAT with standard \textcolor{black}{AT} in Sec. \ref{sec:theoretical_comparison_EGAT_AT}.

\subsubsection{Basic notations and definitions}
\label{sec:pac_notation}

Let \((\mathcal{X}, \mathcal{Y})\) denote the input and output spaces, respectively, and let \(\mathcal{D}\) be an unknown data distribution over \(\mathcal{X} \times \mathcal{Y}\). 
A hypothesis \(f\) maps each input \(x \in \mathcal{X} \subseteq \mathbb{R}^d\) to a probability vector over \(C\) classes,
and $\Phi(x):\mathcal{X}\to\mathbb{R}^d$ is an explanation function. 
We begin by stating the key assumptions on the model $f(x)$ and explanation map $\Phi(x)$ used in our analysis.

\begin{assumption}[Lipschitz Continuity of Model and Explanation]
\label{assumption:lipschitz_continuity}
The model \( f\) is \( \kappa_f \)-Lipschitz with respect to the input under norm \( \| \cdot \|_2 \), and its explanation map \( \Phi \) is \( \kappa_\Phi \)-Lipschitz \cite{bhusaltowards}. That is, for all \( x, x^\prime \in \mathcal{X} \),
\begin{align}
\| f(x) - f(x^\prime) \|_2 &\le \kappa_f \| x - x^\prime \|_2, \\
\| \Phi(x) - \Phi(x^\prime) \|_2 &\le \kappa_\Phi \| x - x^\prime \|_2.
\end{align}
\end{assumption}

\begin{assumption}[Bounded Gradient Magnitude]
\label{assumption:bounded_gradient_magnitude}
    There exists $G > 0$ such that $|\nabla_x f(x)|_2 = |\Phi(x)|_2 \leq G$ for all $x$. This controls explanation sensitivity and is common in interpretability robustness \cite{tan2023robust}.
\end{assumption}

We now define the adversarial loss function used in EGAT and the associated risk quantities.

\begin{definition}[EGAT Adversarial Loss]
Let \( \ell_{cls}(f(x), y) := \mathbf{1}[\hat{y}(x) \ne y] \) be the 0–1 classification loss, where \( \hat{y}(x) = \arg\max_j f_j(x) \) is the predicted class. For a perturbation \( \delta \in \mathbb{R}^d \) with \( \|\delta\| \le \epsilon \), the EGAT adversarial loss is defined as:
\begin{equation}
\begin{aligned}
\ell_{adv}^{EGAT}(x, y; f) := \max_{\|\delta\| \le \epsilon} \bigg(
& \ell_{cls}(f(x+\delta),y) \notag \\
& + \alpha \cdot  \frac{\| \Phi(x + \delta) - \Phi(x) \|_2}{G} \bigg)~,
\end{aligned}
\end{equation}
where \( \alpha > 0 \) is a weighting coefficient, \( G > 0 \) is the explanation normalization constant.
Note that setting \( \alpha = 0 \) recovers the standard adversarial loss.
\end{definition}

\begin{definition}[Population Adversarial Risk under EGAT]
The population-level adversarial risk under EGAT is: 
\begin{equation}
R_{adv}^{EGAT}(f) := \mathbb{E}_{(x,y) \sim \mathcal{D}} [ \ell_{adv}^{EGAT}(x, y; f) ].
\end{equation}
\end{definition}

\begin{definition}[Empirical Adversarial Risk under EGAT]
Given a training dataset \( S = \{(x_i, y_i)\}_{i=1}^n \) drawn i.i.d. from distribution \( \mathcal{D} \), the empirical adversarial risk under EGAT is: 
\begin{equation}
\hat{R}_{adv}^{EGAT}(f) := \frac{1}{n} \sum_{i=1}^n \ell_{adv}^{EGAT}(x_i, y_i; f).
\end{equation}
\end{definition}

\subsubsection{Generalization Bound for EGAT}
\label{sec:generlization_bound_egat}

We now state a generalization bound for the adversarial risk under EGAT. The bound becomes tighter for models with lower Lipschitz constants \( \kappa_f \) and \( \kappa_\Phi \), both of which are implicitly controlled by the EGAT objective.

\begin{theorem}[Generalization Bound of Adversarial Risk for EGAT]
\label{thm:egat_generalization}
Let \( f \in \mathcal{H} \) be a hypothesis satisfying Assumptions~\ref{assumption:lipschitz_continuity} and~\ref{assumption:bounded_gradient_magnitude}. Then, with probability at least \( 1 - \delta \) over the draw of \( n \) i.i.d. samples from \( \mathcal{D} \), the following holds:
\begin{equation}
R_{adv}^{EGAT}(f) 
\leq \hat{R}_{adv}^{EGAT}(f) + \mathcal{O}\left( \frac{(\kappa_f + \kappa_\Phi)\sqrt{d} + \log(1/\delta)}{\sqrt{n}} \right),
\end{equation}
where \( d \) is the input dimensionality and \( n \) is the sample size.
\end{theorem}

\begin{proof}[Proof Sketch]
We outline the main steps leading to the generalization bound.

\textbf{Step 1: Function class definition.}  
Let \( \mathcal{F}_{\text{EGAT}} \) denote the class of loss functions induced by models \( f \in \mathcal{H} \):
\[
\begin{aligned}
\mathcal{F}_{EGAT} = \Big\{ (x, y) \mapsto \sup_{\|\delta\| \le \epsilon} \big[
& \ell_{cls}(f(x+\delta),y) \\
& + \alpha \cdot \tfrac{\|\Phi(x+\delta) - \Phi(x)\|_2}{G} 
\big] \Big\}.
\end{aligned}
\]

\textbf{Step 2: Lipschitz analysis.}  
Using Assumption~\ref{assumption:lipschitz_continuity}, both terms inside the loss are Lipschitz in \( x \), with constants \( \kappa \) and \( \alpha \kappa_\Phi / G \), respectively. Hence, the EGAT adversarial loss is \( (\kappa + \alpha \kappa_\Phi / G) \)-Lipschitz.

\textbf{Step 3: Apply Rademacher complexity bounds.}  
Let $\mathfrak{R}_n(\mathcal{F}_{EGAT})$ denote the empirical Rademacher complexity of the function class. Since the EGAT adversarial loss is $(\kappa_f + \frac{\alpha \kappa_\Phi}{G})$-Lipschitz,
the Rademacher complexity satisfies:
\[
\mathfrak{R}_n(\mathcal{F}_{EGAT}) = \mathcal{O}\left( \frac{(\kappa_f + \frac{\alpha \kappa_\Phi}{G}) \sqrt{d}}{\sqrt{n}} \right).
\]

\textbf{Step 4: Uniform convergence bound.}  
By standard uniform convergence bounds \cite{bartlett2002rademacher},
for losses bounded in $[0,1+\alpha]$, with probability at least $1 - \delta$, we have:
\[
R_{adv}^{EGAT}(f) \le \hat{R}_{adv}^{EGAT}(f) + 2 \mathfrak{R}_n(\mathcal{F}_{EGAT}) + \sqrt{\frac{\log(1/\delta)}{2n}}.
\]
Substituting the bound on complexity yields the stated result. 

\end{proof}

\subsubsection{Theoretical Comparison between EGAT and Standard AT}
\label{sec:theoretical_comparison_EGAT_AT}
We now formalize the comparison between EGAT and traditional \textcolor{black}{AT} under distribution shifts, based on their sensitivity to spurious input components and explanation stability.

\begin{assumption}[Input Decomposition]
\label{assumption:input_decomposition}
    Each input \( x \in \mathbb{R}^d \) can be decomposed as \( x = (x^{(obj)}, x^{(bg)}) \), where \( x^{(obj)} \in \mathbb{R}^{d^{(obj)}} \) encodes causally relevant (objective) features, and \( x^{(bg)} \in \mathbb{R}^{d^{(bg)}} \) denotes background or spurious features.
    The total input dimension satisfies \( d = d^{(obj)} + d^{(bg)} \), and the feature sets are disjoint: $x^{(obj)}\cap x^{(bg)}=\emptyset$.
\end{assumption}

\begin{assumption}
\label{assumption:label_soly_depend_on_obj}
    The label \( y \in \mathcal{Y} \) depends solely on \( x^{(obj)} \), i.e., \( y = g(x^{(obj)}) \) for some unknown function \( g \).
\end{assumption}

\begin{lemma}[Out-of-Distribution Stability via Explanation Consistency]
\label{lemma:ood_explanation_stability}
Let \( x_{train} = (x^{(obj)}, x^{(bg)}_{train}) \) and \( x_{test} = (x^{(obj)}, x^{(bg)}_{test}) \) be two inputs sharing the same core features but differing in background. 
Besides, 
suppose that Assumptions \ref{assumption:lipschitz_continuity}, \ref{assumption:input_decomposition}, and \ref{assumption:label_soly_depend_on_obj} hold.
Then the following bound on the prediction difference holds:
\begin{equation}
\label{eq:ood_output_bound}
\begin{aligned}
\big\| f(x_{test}) - f(x_{train}) \big\| 
& \le \left\| \nabla_{x^{(bg)}} f(x_{train}) \right\| 
   \cdot \| x^{(bg)}_{test} - x^{(bg)}_{train} \| \\
& \quad + \frac{\kappa_\Phi}{2} 
   \cdot \| x^{(bg)}_{test} - x^{(bg)}_{train} \|^2.
\end{aligned}
\end{equation}
\end{lemma}

\begin{proof}[Proof Sketch]
Let \( \Delta x := x_{test} - x_{train} = (0, \Delta x^{(bg)}) \) with \( \Delta x^{(bg)} := x^{(bg)}_{test} - x^{(bg)}_{train} \). Since the variation occurs solely in the background subspace. By the integral form of the mean value theorem,
\begin{equation}
\begin{aligned}
    f(x_{test}) - f(x_{train}) &= \int_0^1 \nabla f(x_{train} + t \Delta x) \cdot \Delta x \, dt. \\
    &= \int_0^1 \nabla_{x^{(bg)}} f(x_{train} + t \Delta x) \cdot \Delta x^{(bg)} \, dt.
\end{aligned}
\end{equation}
Applying Cauchy–Schwarz inequality and using the smoothness of $f$,
\begin{equation*}
\begin{aligned}
    & \left\| f(x_{test}) - f(x_{train}) \right\| \le \\
    & \qquad\qquad\qquad \int_0^1 \left\| \nabla_{x^{(bg)}} f(x_{train} + t \Delta x) \right\| \cdot \left\| \Delta x^{(bg)} \right\| dt.   
\end{aligned}
\end{equation*}
Assuming \( \nabla f \) is \( \kappa_\Phi \)-Lipschitz (as implied by Assumption~\ref{assumption:lipschitz_continuity}), we bound the gradient drift as:
\[
\left\| \nabla_{x^{(bg)}} f(x_{train} + t \Delta x) - \nabla_{x^{(bg)}} f(x_{train}) \right\| \leq \kappa_\Phi t \| \Delta x^{(bg)} \|,
\]
and integration yields the second-order term in Equation~\eqref{eq:ood_output_bound}. This completes the sketch.
\end{proof}

\textbf{Discussion.}
Lemma \ref{lemma:ood_explanation_stability}
indicates that the variation in prediction under background perturbation is controlled by both (1) the model's sensitivity to \( x^{(bg)} \), and (2) the smoothness of the explanation function. 
In particular, define the background sensitivity constant:
$\kappa^{(bg)} := \sup_x \left\| \nabla_{x^{(bg)}} f(x) \right\|$.
We obtain the bound:
\begin{equation}
    \left\| f(x_{test}) - f(x_{train}) \right\| \le \kappa^{(bg)} \cdot \| \Delta x^{(bg)} \| + \frac{\kappa_\Phi}{2} \cdot \| \Delta x^{(bg)} \|^2.
\end{equation}

\textit{EGAT versus Traditional AT.}
Traditional \textcolor{black}{AT} enforces prediction consistency under small norm-bounded input perturbations but does not constrain the model's internal reasoning. 
As a result, while \( f(x_{train}) \approx f(x_{test}) \) may hold, its explanation \( \Phi(x) \) may vary substantially.
In contrast, EGAT explicitly minimizes an attributional loss that penalizes shifts in \( \Phi \) across perturbations. This leads to:
(1) fewer reliance on background features and thus smaller \( \kappa^{(bg)} \);
(2) improved smoothness and lower changes in $\Phi$ and thus smaller $\kappa_{\Phi}$.
Consequently, EGAT-trained models exhibit provably smaller output deviation under domain shifts than standard AT models.

\textit{Unified View of Robustness.}
The above result generalizes across multiple practical settings:
\begin{enumerate}
    \item \textit{Adversarial Robustness:} When $x_{train}$ and $x_{test}$ is a clean input and its adversarial counterpart, 
    Eq.~\eqref{eq:ood_output_bound} shows that EGAT offers stronger guarantees on prediction stability under adversarial perturbations.
    \item \textit{Out-of-Distribution Generalization:} When $x_{train}$ and $x_{test}$ come from different domains but share semantics (e.g., in domain adaptation), the bound similarly implies improved generalization across domains.
\end{enumerate}

\section{Experiments}

\begin{table*}[ht]
\centering
\caption{Performance comparison across domains on VLCS (top) and Terra Incognita (bottom). Metrics include clean accuracy (cAcc), adversarial accuracy (aAcc), Comprehensiveness (Comp), and Sufficiency (Suff).
\textcolor{black}{We highlight \textbf{winner} and \underline{runner-up}.}
}
\label{tab:main_results}
\resizebox{\textwidth}{!}{
\begin{tabular}{l|cccc|cccc|cccc|cccc}
\toprule
\textbf{Model} & \multicolumn{4}{c|}{\textbf{Caltech101}} & \multicolumn{4}{c|}{\textbf{LabelMe}} & \multicolumn{4}{c|}{\textbf{SUN09}} & \multicolumn{4}{c}{\textbf{VOC2007}} \\
& cAcc & aAcc & Comp & Suff & cAcc & aAcc & Comp & Suff & cAcc & aAcc & Comp & Suff & cAcc & aAcc & Comp & Suff \\
\midrule
ERM  
& 98.68 & 38.51 & 0.48 & 38.69 
& 76.17 & 4.24 & 13.97 & 12.78 
& 81.25 & 7.04 & \textbf{18.63} & 27.16 
& 84.88 & 12.96 & 18.47 & 18.77 \\
IGR 
& \underline{99.33} & 71.68 & 1.47 & 39.54 
& 77.33 & 4.00 & 13.21 & \textbf{9.04}
& 79.32 & \underline{12.38} & 14.25 & 31.28 
& 86.59 & \underline{15.55} & 19.94 & \underline{15.72} \\
IGN 
& 99.29	& 76.54 & 2.25 & \underline{13.82} 
& 74.29 & 1.88 & 12.37 & \underline{9.67} 
& \textbf{83.61} & 6.66 & 10.95 & \textbf{25.95} 
& 86.37 & 11.29 & 20.61 & \textbf{11.19} \\
DMADA 
& 99.05 & 69.27 & 0.73 & 47.36 
& 76.41 & 0.93 & 18.29 & 35.62 
& 81.26 & 6.97 & 15.38 & 40.20
& 86.26 & 10.54 & 8.35 & 37.90 \\
IRM
& 95.35 & 43.48 & 2.65 & 30.31 
& 76.41 & 5.22 & 7.16 & 22.14
& 80.46 & 10.07 & 12.23 & 37.43
& 85.22 & 9.60 & 9.05 & 38.77 \\
DRE 
& 99.24 & \underline{77.43} & 3.25 & 29.19 
& \textbf{79.01} & \underline{7.31} & \underline{17.90} & 14.69 
& 80.64 & 5.90 & \underline{18.03} & \underline{26.02}
& 86.07 & 13.14 & 17.02 & 16.04 \\ 
SGDrop 
& 99.31 & 70.35 & \textbf{10.09} & 27.33 
& 78.67 & 6.60 & 15.45 & 10.19 
& 81.04 & 6.09 & 10.56 & 36.19 
& \textbf{87.28} & 13.14 & \underline{21.37} & 26.42 \\
EGAT 
& \textbf{99.66} & \textbf{98.23} & \underline{9.65} & \textbf{1.42} 
& \underline{78.74} & \textbf{55.66} & \textbf{19.71} & 13.76 
& \textbf{83.61} & \textbf{52.19} & 16.40 & 34.19 
& \underline{86.66} & \textbf{45.18} & \textbf{25.70} & 29.64 \\
\bottomrule
\end{tabular}
}

\vspace{1em}

\resizebox{\textwidth}{!}{
\begin{tabular}{l|cccc|cccc|cccc|cccc}
\toprule
\textbf{Model} & \multicolumn{4}{c|}{\textbf{Location 38}} & \multicolumn{4}{c|}{\textbf{Location 43}} & \multicolumn{4}{c|}{\textbf{Location 46}} & \multicolumn{4}{c}{\textbf{Location 100}} \\
& cAcc & aAcc & Comp & Suff & cAcc & aAcc & Comp & Suff & cAcc & aAcc & Comp & Suff & cAcc & aAcc & Comp & Suff \\
\midrule
ERM  
& \textbf{86.02} & 9.82  & 54.86 & 23.59 
& 80.35 & 16.37 & 56.44 & 38.42 
& \underline{79.06} & 3.61  & 43.43 & \underline{15.02} 
& \underline{92.34} & 12.13 & \underline{54.50} & 26.32 \\
IGR  
& 81.83 & \underline{41.93} & 55.30  & 28.61  
& \textbf{81.73} & 8.81  & 51.75  & \underline{39.20} 
& 45.57 & \underline{3.82}  & 12.38 & 28.34 
& 75.31 & 2.90  & 12.78 & 40.51 \\
IGN 
& \underline{85.87} & 36.03 & 54.89 & \textbf{21.49} 
& 81.41 & 13.07 & \underline{60.69} & 37.64 
& \textbf{79.59}	& 3.50 & 48.25 & \textbf{11.95} 
& 91.82 & 8.44 & 49.41 & \textbf{16.49} \\
DMADA
& 80.74 & 29.66 & 29.79 & 53.27
& 80.54 & 7.10 & 47.58 & 82.44
& 73.94 & 3.09 & 36.08 & 51.37
& 92.22 & 3.45 & 33.18 & 62.36 \\
IRM 
& 65.77 & 20.18 & 30.04 & 34.86
& 52.59 & 9.73 & 19.98 & 36.18 
& 60.08 & 3.83 & 20.63 & 35.05
& 83.93 & 4.55 & 7.61 & 35.85 \\
DRE 
& 85.25 & 40.84 & \underline{58.62} & \underline{23.10} 
& 80.47 & \underline{18.42} & \textbf{62.51} & \textbf{37.09} 
& 78.32 & 3.71  & \underline{48.82} & 23.73
& \textbf{93.24} & \underline{24.01} & 36.54 & \underline{25.67} \\
SGDrop 
& 77.80 & 27.87 & 39.35 & 46.11 
& 78.57 & 10.23 & 38.02 & 74.08 
& 76.47 & 3.29 & 36.05 & 49.62 
& 91.19 & 0.52 & 38.93 & 67.71 \\
EGAT
& 85.22 & \textbf{54.52} & \textbf{62.75} & 26.25 
& \textbf{81.73} & \textbf{48.34} & 57.02 & 65.51 
& 78.74 & \textbf{42.40} & \textbf{58.44} & 52.12 
& 91.24 & \textbf{55.27} & \textbf{86.47} & 51.28 \\
\bottomrule
\end{tabular}
}
\end{table*}

We empirically validate EGAT concerning predictive performance, adversarial robustness, out-of-distribution generalization, and explanation quality. 
Comparative analyses are conducted against several sota baselines on benchmark datasets.

\subsection{Datasets and Configurations}
To assess the effectiveness and generalization capability of EGAT, we perform experiments on two \textcolor{black}{OOD} generalization benchmarks following existing works \cite{li2023dre}:
%
    \textbf{Terra Incognita}~\cite{terra} comprises approximately 11{,}000 images categorized into 10 classes, distributed across four domains defined by camera trap locations: \textit{Location 38, 43, 46}, and \textit{100}.
    \textbf{VLCS}~\cite{VLCS} contains around 25{,}000 images across 5 object categories from four domains: \textit{VOC2007, LabelMe, Caltech101}, and \textit{SUN09}.
%
Each sub-dataset is split into three portions by 60\%:20\%:20\%, which will be used for training, validation, and testing, respectively. 
All experiments utilize a ResNet-50 \cite{he2016resnet} architecture initialized with ImageNet-pretrained weights.

\textit{Training Configurations}. Adam optimizer is adopted; learning rate is \(10^{-4}\); batch size is 32; training horizon is 5{,}000 steps; and
$\lambda_1, \lambda_2,\lambda_3, \lambda_4$ are set to 0.5, 1, 1, and 0.2, respectively.
Dropout is applied for all methods to avoid overfitting. 
The checkpoint with the highest validation accuracy is selected for final evaluation on the test set. 
For adversarial attacks, we use PGD with 10 steps 
and $\epsilon=0.02$ in pixel intensity scale by default, 
and more attack methods and settings of $\epsilon$ are further explored in Sec. \ref{sec:exp_adversarial_robustness}. 

\textcolor{black}{
\textit{Runing Environments}.
All experiments were conducted on a workstation with an Intel Xeon E5-2620 v4 CPU (2.10\,GHz) and an NVIDIA Tesla T4 GPU, running CentOS 7.9.2. 
The experiments were executed in a Python 3.9.19, using CUDA toolkit 10.2.89 and PyTorch 1.12.1 (torchvision 0.13.1 and torchaudio 0.9.1). For explanation generation and vision utilities, we used Captum 0.7.0 and OpenCV 4.10.0.84.
}

\subsection{Evaluation Metrics}
In addition to typical adversarial robustness metrics, \textbf{clean accuracy (cAcc)} and \textbf{adversarial accuracy (aAcc)}, we include two metrics to evaluate explanation quality.

\begin{itemize}
    \item \textbf{Comprehensiveness (Comp)}~\cite{guidedlearning} 
    quantifies the importance of identified salient regions by measuring
    the decrease in prediction when these regions are removed: 
    \begin{equation}
        Comp = f_y(x) - f_y(x \setminus \Phi_{:k}(x)),
    \end{equation}
    where $f_y(x)$ is the model's predicted probability for class $y$, and $\Phi_{:k}(x)$ denotes the explanation with the highest importance. Higher $Comp$ indicate greater predictive reliance on the highlighted features.

    \item \textbf{Sufficiency (Suff)}~\cite{guidedlearning} assesses whether the identified explanation alone suffices for the correct prediction:
    \begin{equation}
        Suff = f_y(x) - f_y(\Phi_{:k}(x)).
    \end{equation}
    Lower $Suff$ implies that the explanation retains most of the discriminative information.
\end{itemize}

\subsection{Baselines} 
We compare EGAT with several state-of-the-art baseline methods representing various strategies based on robust training (IGR, IGN, and DMADA) and explanation-guided learning (IRM, DRE, and SGDrop).

\begin{itemize}
    \item \textbf{ERM} \cite{he2016resnet} stands for \textit{empirical risk minimization} solely on classification loss over the training samples.
    \item \textbf{IGR} \cite{ross2018improving} adopts an input gradient regularization exhibiting robustness to transferred adversarial examples.
    \item \textbf{IGN} \cite{chen2019ignorm} achieves robust Integrated Gradients attributions \cite{sundararajan2017ig} to promote generalization.
    \item \textbf{DMADA} \cite{xu2020dmada} designs an \textcolor{black}{AT} framework mapping various domains to a common latent space. 
    \item \textbf{IRM} \cite{arjovsky2019irm} learns predictors invariant across environments by enforcing consistent gradients.
    \item \textbf{DRE} \cite{li2023dre} incorporates uncertainty quantification into \textcolor{black}{EGL} training phase.
    \item \textbf{SGDrop} \cite{bertoin2024sgdrop} selectively drops the \textit{most} salient features during training to improve generalization.
\end{itemize}

\subsection{Main results}
\begin{figure*}
    \centering
    \includegraphics[width=0.98\linewidth]{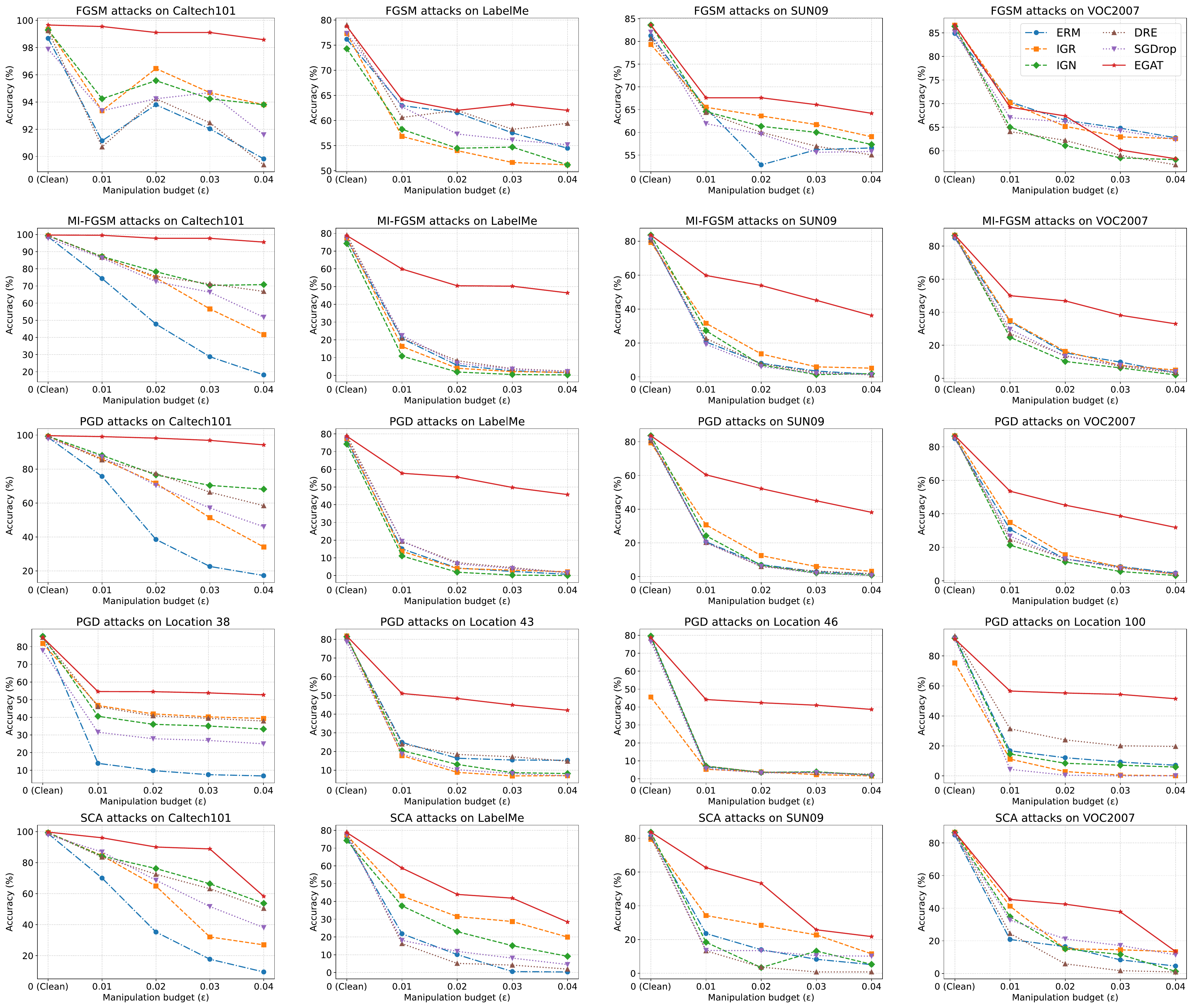}
    \caption{Adversarial robustness across different attack methods (FGSM, MI-FGSM, PGD, \textcolor{black}{SCA}) and datasets (VLCS and Terra Incognita). The figures show accuracy under varying perturbation budgets, ranging from 0 (clean) to 0.04. 
    EGAT (in red solid curves) shows significantly more stable predictions than  baselines under strong adversarial attacks.} 
    \label{fig:adv_robustness_all}
\end{figure*}

Table \ref{tab:main_results} reports \textcolor{black}{cAcc} and \textcolor{black}{aAcc} concerning \textit{prediction performance}, \textcolor{black}{Comp} and \textcolor{black}{Suff} concerning \textit{explanation quality}. 
In short, compared with six baselines,
EGAT achieves superior adversarial robustness while maintaining high standard accuracy and high explanation quality. 
Specifically, we have the following observations.

\paragraph{Standard Accuracy}
Across both VLCS and Terra Incognita benchmarks, EGAT consistently matches or exceeds the cAcc of state-of-the-art baselines, achieving the best performance on Caltech101 (99.66\%), SUN09 (83.61\%), and Location 43 (81.73\%), and has comparable performance in other domains. 
This demonstrates that the addition of explanation-guided constraints does not compromise general predictive performance. 
Their predictive performance (cAcc) under out-of-distribution scenarios is discussed in Sec. \ref{sec:exp_ood_generalization}.

\paragraph{Adversarial Robustness}
EGAT demonstrates a substantial improvement in aAcc over \textbf{all} baselines. For instance, on Caltech101, EGAT achieves an aAcc of 98.23\%, far surpassing DRE (77.43\%) and IGR (71.68\%). 
Similarly, on the Terra Incognita benchmark, EGAT consistently delivers the highest adversarial robustness across all locations.
It is worth noting that on LabelMe and Location 46, all baselines cannot achieve 10\% accuracy under attacks, while EGAT maintains 55.66\% and 42.40\%, respectively, highlighting its effectiveness in defending against adversarial perturbations. 
These gains suggest that adversarial regularizations lead the model to base decisions on more robust, causally relevant features rather than brittle spurious correlations. 
We will further evaluate their adversarial robustness under various manipulation methods and strengths in Sec. \ref{sec:exp_adversarial_robustness}.

\paragraph{Explanation Quality}
Notably, EGAT consistently yields the highest or near-highest Comp scores across many settings, suggesting that the model effectively avoids reliance on irrelevant features, such as spurious background cues, during decision-making. 
For example, on Location 46 and Location 100 in Terra Incognita, EGAT achieves Comp scores 58.44 and 86.47, respectively, demonstrating its strong reliance on causally relevant regions for prediction. 

Simultaneously, EGAT also attains relatively low Suff scores compared to most baselines (e.g., 1.42 on Caltech101 and 13.76 on LabelMe) in some domains. 
However, there are instances where EGAT also produces high Suff scores, which may indicate cases where the model captures features within the object that are not the most discriminative for the task at hand. 
For example, while EGAT may identify the paws of a cat, these may contribute less significantly to classification compared to more salient features like the cat's fur texture or facial structure.

These observations highlight a potential future direction for improvement: enhancing EGAT's focus on the most informative subregions within causally relevant objects to further strengthen interpretability and predictive sufficiency.

\paragraph{Limitations} Despite its strong overall performance, especially predictive accuracy under adversarial attacks,
EGAT does not always achieve the best score across all individual metrics or domains. 
For example, in a few cases, models like DRE or IGR outperform EGAT on specific interpretability metrics. 
We hypothesis it for two reasons: 
(1) One possible reason is that EGAT introduces additional explanation-guided losses that may trade off with raw predictive accuracy in edge cases where spurious features happen to be predictive within the training domain. 
(2) Moreover, EGAT's reliance on saliency-based attributions might be affected by noise or local non-linearity in deep networks, potentially limiting the quality of its guidance in complex scenes. 

For the following experiments, we include all baselines from Table~\ref{tab:main_results} except DMADA and IRM, as these two methods do not demonstrate strong performance across domains and metrics.

\subsection{Robustness to Adversarial Attacks}
\label{sec:exp_adversarial_robustness}

\begin{table*}[t]
\centering
\caption{Clean accuracy and corresponding ranking (in brackets) under OOD evaluation across different train-test settings.
The average and variance of rankings across various train-test settings are disclosed in the last column.}
\resizebox{\textwidth}{!}{
\begin{tabular}{l|ccc|ccc|ccc|ccc|c}
\toprule
\textbf{Train} on 
& \multicolumn{3}{c|}{Caltech101} 
& \multicolumn{3}{c|}{LabelMe} 
& \multicolumn{3}{c|}{SUN09} 
& \multicolumn{3}{c|}{VOC2007}
& Ranking \\
\textbf{Test} on
& LabelMe & SUN09 & VOC2007 
& Caltech101 & SUN09 & VOC2007 
& Caltech101 & LabelMe & VOC2007 
& Caltech101 & LabelMe & SUN09 
& Avg. $\pm$ Var. \\
\midrule
ERM 
& 28.11 [5] & 34.66 [5] & 48.01 [2]
& 91.59 [4] & 56.40 [4] & 66.55 [5]
& 73.45 [2] & 57.54 [6] & 63.14 [5] 
& 94.34 [6] & 52.64 [6] & 71.49 [4] 
& 4.50 ${\scriptstyle \pm 1.75}$\\

IGR 
& 33.01 [3] & 43.04 [1] & 52.96 [1]
& 86.28 [6] & 54.28 [6] & 64.44 [6]
& 69.91 [3] & 60.37 [5] & 58.14 [6]
& 97.87 [2] & 62.50 [5] & 79.80 [1] 
& 3.75 ${\scriptstyle \pm 4.19}$ \\

IGN 
& 26.65 [6] & 33.52 [6] & 43.70 [5]
& 88.49 [5] & 55.61 [5] & 70.18 [1]
& 67.25 [4] & 64.38 [3] & 63.51 [4]
& 97.34 [5] & 64.15 [4] & 76.57 [2] 
& 4.17 {$\scriptstyle \pm 2.14$} \\

DRE 
& 30.94 [4] & 34.75 [4] & 45.22 [4]
& 93.36 [2] & 62.34 [2] & 67.40 [3]
& 34.75 [6] & 62.34 [4] & 70.48 [2]
& 97.58 [4] & 67.40 [1] & 61.45 [6]
& 3.50 {$\scriptstyle \pm 2.25$} \\

SGDrop 
& 34.19 [2] & 40.85 [2] & 42.96 [6] 
& 92.47 [3] & 64.63 [1] & 67.22 [4]
& 40.85 [5] & 64.63 [2] & 75.76 [1]
& 98.67 [1] & 67.22 [2] & 61.85 [5] 
& 2.83 {$\scriptstyle \pm 2.81$} \\

EGAT 
& 38.15 [1] & 39.04 [3] & 46.59 [3] 
& 94.34 [1] & 57.33 [3] & 69.77 [2]
& 74.20 [1] & 67.68 [1] & 66.22 [3] 
& 97.87 [2] & 64.62 [3] & 74.66 [3]
& \textbf{2.17} {$\scriptstyle \pm 0.81$} \\

\bottomrule
\end{tabular}
}
\label{tab:ood_acc_rearranged}
\end{table*}

To further analyze EGAT and other baselines adversarial robustness, 
\textcolor{black}{we conduct four different types of attacks, including FGSM \cite{goodfellow2014fgsm}, MI-FGSM \cite{dong2018mifgsm}, PGD \cite{madry2017pgd}, and SCA \cite{pan2024sca}.}
The manipulation budget $\epsilon$ ranges from 0.01 to 0.04.
Based on the results reported in Fig. \ref{fig:adv_robustness_all}, 
we have the following analysis.

\textcolor{black}{FGSM}
\cite{goodfellow2014fgsm} is a one-step gradient-based attack. 
In our experiments, all models suffer only moderate drops under FGSM, where most models can keep at least 50\% adversarial accuracy under the severest FGSM manipulations. 
Compared to baselines, EGAT consistently retains the highest accuracy at almost every perturbation level. 
Especially in Caltech101, EGAT maintains nearly clean-level accuracy even when $\epsilon=0.04$, whereas baselines like ERM
and DRE fall off fast as $\epsilon$ grows. 
This suggests EGAT's learned features are less sensitive to the single-step noise, likely because its training emphasizes invariant causal cues rather than spurious correlations. 
In contrast, other methods rely more on superficial patterns that even a weak FGSM can exploit. 

\textcolor{black}{MI-FGSM}
\cite{dong2018mifgsm} injects momentum into iterative FGSM, making it a much stronger adversary. 
In line with prior findings, 
this attack causes far steeper accuracy drops than FGSM, 
where most baselines have plunged nearly to zero adversarial accuracy when $\epsilon$ is 0.03.
In contrast, EGAT outperforms all baselines across all datasets on VLCS. 
Specifically, in the case of $\epsilon=0.03$, EGAT retains around at least half the accuracy of the clean model across datasets. 
The robustness trend between MI-FGSM and FGSM aligns with previous work \cite{dong2018mifgsm} showing MI-FGSM's known potency: ``integrating momentum, MI-FGSM outperforms FGSM significantly''.

\textcolor{black}{PGD}
\cite{madry2017pgd} is an iterative multi-step attack widely regarded as one of the most potent first-order methods. 
On VLCS, PGD induces steep accuracy declines for all baseline methods. 
Nevertheless, EGAT remains far more robust than any competitor. 
The robustness trend is quite similar to those under MI-FGSM attacks. 
%
A similar pattern is observed on Terra Incognita (Location 38/43/46/100). 
EGAT yields around 10–20\% accuracy at high $\epsilon$, whereas most baselines are essentially at 0\%. 
These results highlight that EGAT’s learned features withstand even very strong iterative perturbations by grounding predictions in invariant cross-domain cues. 

\textcolor{black}{
SCA \cite{pan2024sca} leverages DDPM inversion to improve semantic consistency in adversarial manipulations. 
Compared to gradient-only attacks, SCA can generate perturbations that are more semantically coherent and thus can expose different vulnerabilities than $\ell_p$-bounded optimization alone.
SCA leads to substantial degradation for all methods as $\epsilon$ increases, while EGAT consistently achieves the best or near-best performance across datasets and budgets.
This result suggests that EGAT improves resilience to both standard first-order attacks and semantically consistent perturbations.
}

In summary, while PGD \textcolor{black}{and SCA} attacks drastically reduce performance for everyone (as theory predicts for such an aggressive attack), EGAT consistently outperforms other methods. Its robustness under various attacks is the highest across datasets, validating that the proposed EGAT architecture confers significant adversarial resilience.
Notably, however, no model is fully immune: as $\epsilon$ increases under PGD or MI-FGSM, EGAT's accuracy also drops dramatically (albeit more slowly), underscoring that extremely powerful attacks can eventually overcome domain-generalization defenses.

\subsection{Generalization under Out of Distribution}
\label{sec:exp_ood_generalization}

To evaluate the models’ capability to generalize under distribution shifts, Table~\ref{tab:ood_acc_rearranged} reports the \textcolor{black}{cAcc} of EGAT and baseline methods under \textcolor{black}{OOD} settings on VLCS, where the training and testing domains differ while label spaces remain consistent. 
To provide a holistic view of generalization ability, we also report each model's rank per setting (in brackets), and the mean and variance of rankings across all cases.

\paragraph{Overall OOD Generalization Trends}
EGAT demonstrates consistently strong OOD generalization across most train-test pairs. 
For instance, when trained on \textit{Caltech101} and tested on \textit{LabelMe}, EGAT achieves 38.15\%, outperforming the runner-up SGDrop (34.19\%) by an 11.6\% relative improvement. 
Similarly, when trained on \textit{SUN09} and tested on \textit{LabelMe}, EGAT attains 67.68\%, outperforming IGN (64.38\%) with a 5.1\% relative gain. 
Across all settings, EGAT secures the best average rank (2.17), significantly ahead of the next best SGDrop (2.83), indicating superior adaptability to distributional shifts.
These gains align with the design intuition of EGAT: by enforcing explanation consistency during training, the model is encouraged to rely on causally relevant, semantically aligned, and domain-invariant features, which aid generalization to unseen target domains under covariate shifts. 

\paragraph{Consistency across Different Settings}
Besides EGAT demonstrates overall strong OOD generalization, an examination of \textit{rankings} across different settings reveals its remarkable \textit{consistency}. 
As shown in Table~\ref{tab:ood_acc_rearranged}, EGAT consistently ranks within the \textit{top three} positions across all train-test settings, never dropping below rank 3 in any scenario. 
The runner-up SGDrop achieves a competitive average rank of 2.83, and could beat EGAT in few times. 
However, SGDrop exhibits significant fluctuations, falling to rank 5 or 6 in several settings (e.g., \textit{Caltech101} $\rightarrow$ \textit{SUN09}, \textit{VOC2007} $\rightarrow$ \textit{SUN09}, and \textit{SUN09} $\rightarrow$ \textit{Caltech101}). 
The variance of rankings delivers a straightforward comparison: 
EGAT consistently holds the top place (with a variance of 0.81), while SGDrop generally goes up and down (with a variance of 2.81).
Similarly, other baselines such as IGR and DRE show considerable variability, oscillating between top-three ranks in some cases while easily dropping to lower ranks (5 or 6) in others.

This ranking stability of EGAT underscores its robustness and adaptability under diverse domain shifts.
This consistency aligns with Lemma~\ref{lemma:ood_explanation_stability}, illustrating that enforcing explanation-guided consistency fosters reliance on stable, semantically meaningful features that generalize reliably across domains, even under covariate shifts.

\paragraph{Limitations and Failure Cases}
\textcolor{black}{
While EGAT attains the most consistent overall performance across train-test domain shifts (average ranking $2.17$ with variance $0.81$), there exist specific OOD settings where another baseline achieves higher clean accuracy. 
We hypothesize two main reasons: 
(1) EGAT relies on saliency-based attributions that may be noisy or unstable in complex scenes such as cluttered backgrounds, which can weaken the reliability of the explanation-guidance signal and occasionally amplify spurious correlations. 
We further explore its potential solution in Sec. \ref{tab:sensitivity_analysis}.
(2) EGAT encourages explanation consistency and alignment across perturbations, which may reduce model flexibility under domain shifts where the \emph{causal features} associated with a label vary substantially across domains. 
In such cases, Assumptions \ref{assumption:input_decomposition} and \ref{assumption:label_soly_depend_on_obj} may be violated, and enforcing a stable attribution pattern can bias the model toward features that are not transferable across domains. 
}


\subsection{Sensitivity Analysis}
{\color{black}
To further study the failure cases in Sec. \ref{sec:exp_ood_generalization}, we focus on the second hypothesis regarding ``noisy saliency'' issue. 
We conduct two sets of analyses: the effect of \textit{explanation methods}, and the effect of the \textit{guide-map quality}.

\paragraph{Effect of explanation method}
Since EGAT is agnostic to the specific attribution mechanism, we examine the impact of the explanation method by comparing EGAT(base), which uses Grad-CAM, against three variants that replace the explanation module with SmoothGrad (SG)~\cite{smilkov2017smoothgrad}, Integrated Gradients (IG)~\cite{sundararajan2017ig}, and Finer-CAM~\cite{zhang2025finercam}, respectively.
As shown in Table~\ref{tab:sensitivity_analysis}, EGAT with SG improves OOD accuracy on LabelMe (38.15$\rightarrow$40.51), SUN09 (39.04$\rightarrow$44.50), and VOC2007 (46.59$\rightarrow$50.20) compared to EGAT (base), and similar improvements can also be observed for IG and Finer-CAM. 
These results suggest that EGAT remains effective across various attribution methods and can benefit from explanation methods that yield less noisy saliency maps.

\paragraph{Effect of guide-map quality}
We further evaluate EGAT's sensitivity to the quality of the guide map $M(x)$.
We fix Grad-CAM as the default explanation module and perturb $M(x)$ by injecting Gaussian noise with two levels, i.e., $N(0,0.01^2I)$ and $N(0,0.5^2I)$.
Table~\ref{tab:sensitivity_analysis} shows that mild noise ($N(0,0.01^2I)$) has only a limited impact, as performance remains close to EGAT (base), whereas heavier noise ($N(0,0.5^2I)$) leads to a clear performance drop.
Finally, we report an oracle-quality setting by replacing $M(x)$ with human-labeled (``golden'') guides.
This substantially improves OOD accuracy (e.g., SUN09: 39.04$\rightarrow$45.44; VOC2007: 46.59$\rightarrow$56.80), demonstrating that higher-quality guidance can further strengthen EGAT and providing an empirical upper bound when reliable semantic guides are available.

\begin{table}[ht]
\centering
{\color{black}
\caption{
Sensitivity analysis under OOD evaluation: models are trained on VOC2007 and tested on all others. EGAT (base) uses Grad-CAM; SG/IG/Finer-CAM replace the explanation module. The last three rows vary the guide map quality.
}
\label{tab:sensitivity_analysis}
\begin{tabular}{l|c|c|c|c}
\toprule
\textbf{Model} & \textbf{Caltech101} & \textbf{LabelMe} & \textbf{SUN09} & \textbf{VOC2007} \\
\midrule
IGR & 99.33 & 33.01 & 43.04 & 52.96 \\
SGDrop & 99.31 & 34.19 & 40.85 & 42.96 \\
\midrule
EGAT (base) & 99.66 & 38.15 & 39.04 & 46.59 \\
\textit{w} SG & 99.54 & 40.51 & 44.50 & 50.20 \\
\textit{w} IG & 99.68 & 35.79 & 42.44 & 50.94 \\
\textit{w} Finer-CAM & 99.47 & 38.49 & 44.44 & 45.16 \\
\midrule
\textit{w} $N(0, 0.01^2I)$ & 99.63 & 38.02 & 38.58 & 46.55 \\
\textit{w} $N(0, 0.5^2I)$ & 99.29 & 36.86 & 37.42 & 40.96 \\
\textit{w} golden & 99.76 & 40.19 & 45.44 & 56.80 \\
\bottomrule
\end{tabular}
}
\end{table}
}

\subsection{Ablation Study}
\label{sec:ablation_study}

\begin{table*}[ht]
\centering
\caption{Ablation results on \textbf{VLCS} (top) and \textbf{Terra Incognita} (bottom). 
}
\label{tab:ablation_study}
\resizebox{\textwidth}{!}{
\begin{tabular}{l|cccc|cccc|cccc|cccc}
\toprule
\textbf{Model} & \multicolumn{4}{c|}{\textbf{Caltech101}} & \multicolumn{4}{c|}{\textbf{LabelMe}} & \multicolumn{4}{c|}{\textbf{SUN09}} & \multicolumn{4}{c}{\textbf{VOC2007}} \\
& cAcc & aAcc & Comp & Suff & cAcc & aAcc & Comp & Suff & cAcc & aAcc & Comp & Suff & cAcc & aAcc & Comp & Suff \\
\midrule
ERM & 98.68 & 38.51 & 0.48  & 38.69 & 76.17 & 4.24  & 13.97 & 12.78 & 81.25 & 7.04  & 18.63 & 27.16 & 84.88 & 12.96 & 18.47 & 18.77 \\
\textit{w/o} EGL & 99.64 & 97.34 & 3.40 & 23.68 & 74.05 & 55.66 & 26.28 & 8.23 & 82.15 & 55.23 & 23.49 & 25.22 & 86.29 & 47.59 & 24.81 & 19.29 \\
\textit{w/o} AT & 99.98 & 63.27 & 28.53 & 28.22 & 80.18 & 11.32 & 14.02 & 14.82 & 82.58 & 7.42  & 19.90 & 25.62 & 86.90 & 17.22 & 19.21 & 15.15 \\
EGAT & 99.66 & 98.23 & 9.65  & 1.42  & 78.74 & 55.66 & 19.71 & 13.76 & 83.61 & 52.19 & 16.40 & 34.19 & 86.66 & 45.18 & 25.70 & 29.64 \\
\bottomrule
\end{tabular}
}

\vspace{1em}

\resizebox{\textwidth}{!}{
\begin{tabular}{l|cccc|cccc|cccc|cccc}
\toprule
\textbf{Model} & \multicolumn{4}{c|}{\textbf{Location 38}} & \multicolumn{4}{c|}{\textbf{Location 43}} & \multicolumn{4}{c|}{\textbf{Location 46}} & \multicolumn{4}{c}{\textbf{Location 100}} \\
& cAcc & aAcc & Comp & Suff & cAcc & aAcc & Comp & Suff & cAcc & aAcc & Comp & Suff & cAcc & aAcc & Comp & Suff \\
\midrule
ERM & 86.02 & 9.82  & 54.86 & 23.59 & 80.35 & 16.37 & 56.44 & 38.42 & 79.06 & 3.61  & 43.43 & 15.02 & 92.34 & 12.13 & 54.50 & 26.32 \\
\textit{w/o} EGL & 82.53 & 54.33 & 63.49	& 30.81 & 81.10 & 49.60 & 60.45 & 41.92 & 79.48 & 42.82 & 37.73	& 33.12 & 91.16 & 62.40 & 59.66 & 43.17 \\
\textit{w/o} AT & 87.19 & 36.22 & 63.70 & 20.10 & 83.75 & 18.89 & 67.58 & 35.64 & 79.48 & 4.35 & 49.52 & 8.84  & 93.88 & 26.64 & 55.18 & 18.67 \\
EGAT & 85.22 & 54.52 & 62.75 & 26.25 & 81.73 & 48.34 & 57.02 & 65.51 & 78.74 & 42.40 & 58.44 & 52.12 & 91.24 & 55.27 & 86.47 & 51.28 \\
\bottomrule
\end{tabular}
}
\end{table*}

To understand the individual contributions of each component within EGAT, we conduct a systematic ablation study focusing on its two core objectives: adversarial training loss (\(\mathcal{L}_{adv}\)) and explanation-guided loss (\(\mathcal{L}_{egl}\)). 
To isolate the effects of these components, we retain the classification loss ($\mathcal{L}_{cls}$) and the regularization term ($\mathcal{L}_{reg}$) for unchanged throughout this study. 
Specifically, we evaluate the following variants of EGAT:

\begin{itemize}
    \item \textbf{ERM (w/o EGL \& AT):} Standard empirical risk minimization without \textcolor{black}{AT} nor \textcolor{black}{EGL} components.
    \item \textbf{w/o EGL:} EGAT without \textcolor{black}{EGL} loss $\mathcal{L}_{egl}$ ($\lambda_2=0$).
    
    \item \textbf{w/o AT:} EGAT without \textcolor{black}{AT} loss $\mathcal{L}_{at}$ ($\lambda_1=0$).
    \item \textbf{EGAT (full):} The proposed method using both \textcolor{black}{AT} and \textcolor{black}{EGL} loss.
\end{itemize}

Following settings in Table \ref{tab:main_results}, we evaluate each ablation version on both the VLCS and Terra Incognita benchmarks using \textcolor{black}{cAcc}, \textcolor{black}{aAcc}, \textcolor{black}{Comp}, and \textcolor{black}{Suff} to comprehensively assess predictive performance, robustness, and explanation quality. 

Based on results reported on Table \ref{tab:ablation_study}, we have the following observations and conclusions: 

\paragraph{Contribution of Explanation Guided Loss} 
Comparing w/o EGL with full EGAT reveals that the removal of explanation guidance may lead to higher adversarial accuracy, but at the cost of explanation quality. 
For example, on Location 46, w/o EGL attains 42.82\% aAcc (slightly higher than EGAT's 42.40\%) but shows a significantly lower Comp score (37.73 vs. 58.44 for EGAT), indicating less effective use of causally relevant features during prediction. 
This trade-off highlights that while \textcolor{black}{AT} alone can enhance robustness, \textcolor{black}{EGL} encourages the model to rely on semantically meaningful features, which may introduce slight compromises in adversarial performance but significantly improves interpretability.

\paragraph{Effectiveness of Adversarial Training} 

\textcolor{black}{
The AT-only variant (w/o EGL) can attain adversarial robustness comparable to EGAT in some cases, but typically sacrifices explanation quality. 
For example, on Caltech101, w/o EGL achieves 97.34\% aAcc versus 98.23\% for EGAT, yet EGAT yields a substantially higher Comp (9.65 vs. 3.40) and lower Suff (1.42 vs. 23.68), suggesting that EGAT more consistently bases its decisions on salient and semantically meaningful regions. 
These results indicate that when explanations are not required downstream, the AT-only baseline may be a reasonable choice, since it avoids the additional saliency-consistency computations introduced by $\mathcal{L}_{egl}$ while still providing strong robustness in adversarial environments.
}

\textcolor{black}{
Conversely, in \emph{safety-critical or audit-driven applications} (e.g., medical imaging decision support \cite{liu2024xfmp}, financial system debugging \cite{weber2024applications}), EGAT should be preferred because robustness alone is insufficient: stable and faithful explanations are essential for human verification, error analysis, and regulatory scrutiny.
Moreover, beyond improved interpretability, EGAT exhibits stronger robustness than multiple AT baselines under OOD scenarios (as discussed above), reinforcing its suitability for deployments where both distribution shift and accountability are central concerns.
}

\paragraph{Summary}
This ablation confirms that EGAT's strength lies in seamlessly integrating \textcolor{black}{AT} and explanation guidance, leading to models that are not only robust but also interpretable under challenging conditions.
These insights guide practical deployment: practitioners prioritizing robustness and interpretability in sensitive domains, such as medical imaging and autonomous driving, will benefit most from the full EGAT setup, while understanding the potential trade-offs introduced by each component.

\subsection{Runtime Analysis}
{\color{black}
We further compare EGAT with baselines concerning computational overhead. 
Unless otherwise specified, all methods adopt the same backbone (ResNet-50) and inference pipeline; therefore, the \emph{inference} cost for validation/test is identical across methods, and the main computational difference lies in \emph{training}. 
In addition to EGAT with ResNet-50, we also consider two lightweight variants by replacing the backbone with MobileNetV2 \cite{sandler2018mobilenetv2} and EfficientNet \cite{tan2019efficientnet}, respectively.

Fig.~\ref{fig:time_vs_aacc} reports aAcc versus per-epoch training time (averaged over 100 epochs). 
We highlight the Pareto-optimal models which are not dominated in terms of both robustness and training cost.
Overall, EGAT-ResNet delivers higher robustness at a moderate training-time increase. 
For example, EGAT-ResNet achieves an improvement of $+$37\% on aAcc requiring only a limited increase ($+16\%$) in per-epoch training time.
In contrast, several baselines (e.g., IRM and IGN) incur considerable training overhead but still exhibit markedly lower aAcc, indicating that additional computation alone does not guarantee robustness improvements.
Meanwhile, EGAT-MobileNetV2 and EGAT-EfficientNet further reduce the per-epoch training time and lie on (or near) the Pareto front, offering a practical option when training budgets are constrained.
}

\begin{figure}[t]
    \centering
    \includegraphics[width=0.95\linewidth]{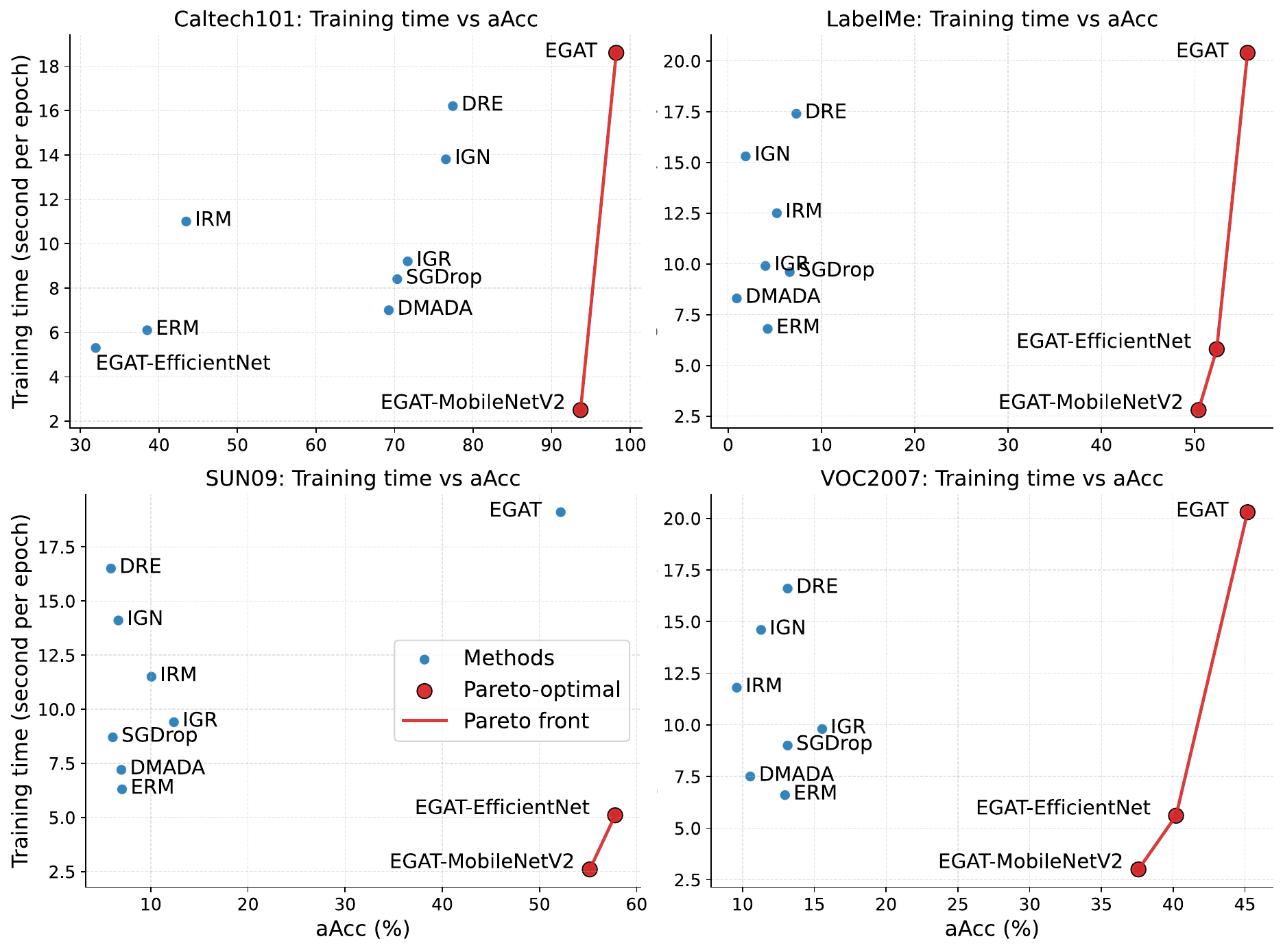}
    \caption
    {Adversarial accuracy \textit{vs.} training time per epoch. Red points denote Pareto-optimal methods and the red curve denotes the Pareto front.}
    \label{fig:time_vs_aacc}
\end{figure}

\subsection{Case Study: Explanation Quality}
\label{sec:case_study}

\begin{figure}[t]
    \centering
    \includegraphics[width=0.95\linewidth]{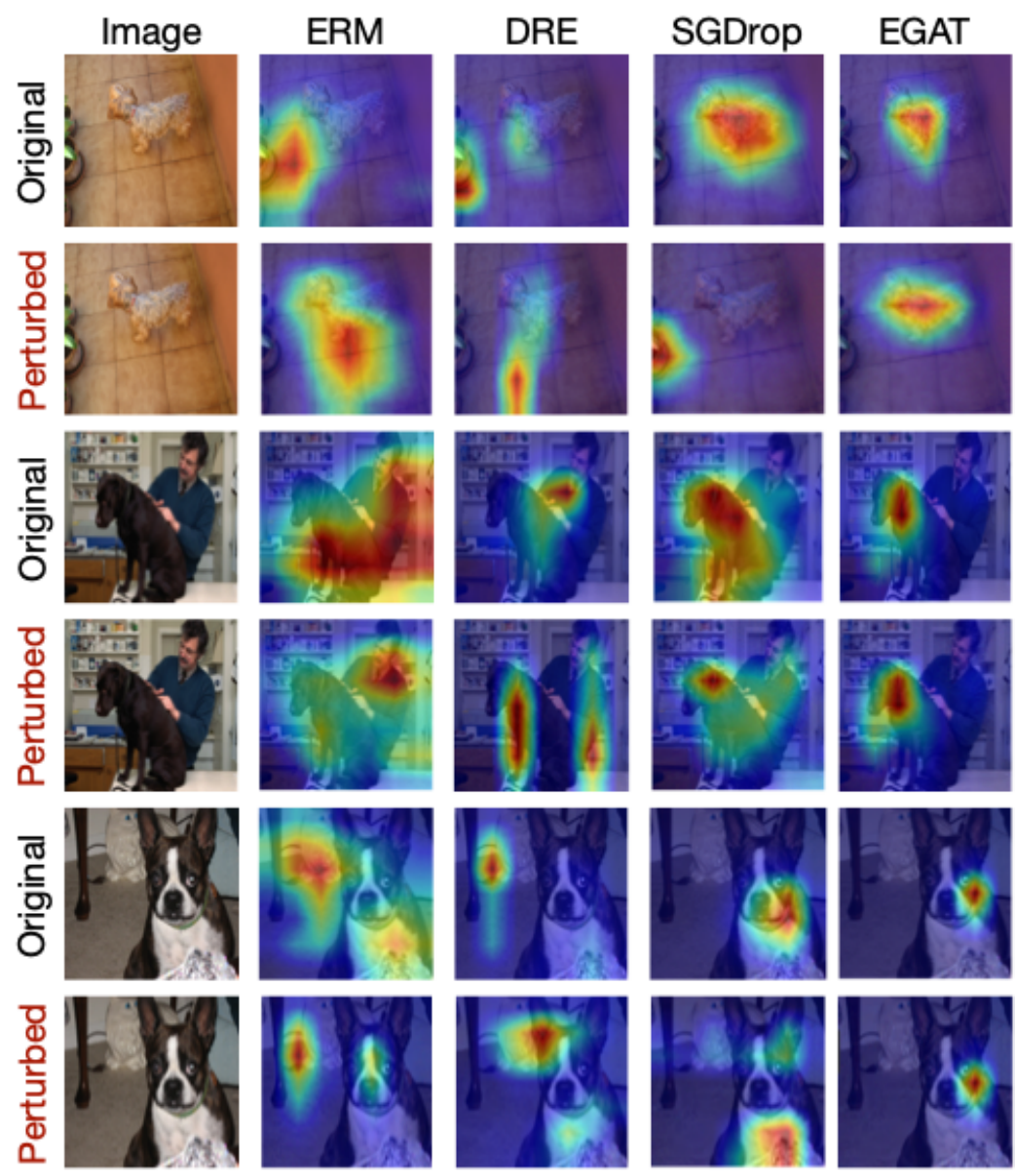}
    \caption
    {Visualization of explanation heatmaps under clean and adversarial conditions.
    We present Grad-CAM heatmaps of four models (columns: ERM, DRE, SGDrop, and EGAT) for three images before and after PGD perturbations. 
    These visualizations illustrate where each model attends when making predictions and how their explanations change under adversarial attacks. }
    \label{fig:case_study}
\end{figure}

Table~\ref{tab:main_results} quantitatively evaluates explanation quality using comprehensiveness and sufficiency metrics. 
To complement these results, we present three qualitative case studies from VLCS in Fig. \ref{fig:case_study}, providing intuitive visual comparisons. 
For each example, we display Grad-CAM heatmaps generated by each model both before and after adversarial perturbations, illustrating the regions each model prioritizes during prediction. 
Specifically, Fig. \ref{fig:case_study} compares ERM, DRE, SGDrop, and EGAT, revealing their respective attention distributions. 

To exemplify these observations, we discuss the first image containing a small dog on the floor (top two rows) in detail. 
PGD attacks \cite{madry2017pgd} are applied to each image to assess explanation stability and robustness alignment under adversarial conditions, providing a comprehensive assessment of interpretability under realistic threat models.

\paragraph{Explanation Heatmap for Original Image}
Heatmaps of ERM and DRE are fragmented and misaligned, with attention scattered on irrelevant regions of the floor or background.
These methods often fail to capture the dog as the central causal feature, indicating a reliance on spurious correlations.
The heatmap of SGDrop shows tight, concentrated attention on the dog's body, correctly identifying the object as the main feature for prediction. 
EGAT produces a even more compact and precise heatmap centered on the dog.
Both heatmaps demonstrate strong clean interpretability, aligning well with human understanding.

\paragraph{Explanation Heatmap Analysis under Adversarial Perturbations}
Under adversarial perturbations,
the heatmaps of ERM and DRE become even more scattered, often focusing on irrelevant tiles or edges unrelated to the object.
While providing clean focus under clean settings, SGDrop's heatmap collapses under adversarial noise, losing object focus and highlighting irrelevant regions. 
Remarkably, EGAT maintains a consistent, object-centered heatmap even under adversarial perturbations.
The explanation remains focused on the dog's torso, demonstrating robust interpretability and resilience to adversarial noise.

\paragraph{Trends are Consistent across Images}
For the second image (small dog on the tiled floor) and the third image (dog on examination table with a person), consistent trends are observed across methods. 
ERM and DRE frequently distract to background, such as the cat in the background or the examiner, failing to capture the object reliably under both clean and adversarial settings. 
EGAT and SGDrop again shows tighter, object-centric attention, focusing precisely on the dog’s body while ignoring the tile background, aligning better with human-intuitive explanations.
More importantly, 
EGAT uniquely enforces semantic consistency, resulting in stable and meaningful heatmaps aligned with object regions critical for decisions even facing adversarial attacks.

\section{Conclusion}
\subsection{Limitations and Discussion}
While the theoretical insights and empirical results presented in this work demonstrate the effectiveness of our approach in enhancing both robustness and interpretability, several limitations merit discussion: 

{\color{black}
\textbf{Dependence on explanation method.}
To ensure comparability with prior work, the main implementation of EGAT adopts Grad-CAM as the default explanation module, and we additionally evaluate representative gradient-based attribution methods (e.g., SmoothGrad, Integrated Gradients, and Finer-CAM). 
%
For \textit{non} gradient-based explanation paradigms (e.g., counterfactual explanations \cite{goyal2019counterfactual}), EGAT can be extended in several ways.
(1) For counterfactual-style explanations that output a minimal edit $\delta$ (rather than a saliency map), the stability objective can be reformulated to penalize the discrepancy between counterfactual edits for $x$ and $x^{adv}$ (e.g., $\|\delta -\delta^{adv}\|$). 
(2) As for other intractable explanation methods, the explanation operator can be replaced by a \emph{surrogate differentiable} approximation \cite{ribeiro2016lime}, such as a learned explainer network trained to mimic the target explanation signal. EGAT can then optimize robustness while enforcing consistency in the surrogate explanation space. 

\textbf{Sensitivity to guide-map quality.}
As shown in sensitivity analysis, degrading the guide map can reduce the benefit of EGAT, and under severe domain shifts the constraint may become over-restrictive when the causal evidence for a class varies across domains. 
A direct remedy is to make the explanation-guided term \emph{adaptive} rather than using a fixed weighting throughout training.
For example,
down-weighting $\mathcal{L}_{egl}$ when the model exhibits high uncertainty on perturbed samples \cite{patro2019u}, thereby preventing the training objective from enforcing potentially unreliable attribution patterns.
}

\subsection{Summary and Future Directions}

In this paper, we introduced Explanation-Guided Adversarial Training, a unified training framework that integrates adversarial robustness with explanation-guided learning. 
%
Motivated by theoretical considerations and supported by comprehensive empirical studies, EGAT demonstrates strong performance across two domain generalization benchmarks, achieving high adversarial accuracy, stable out-of-distribution generalizability, and producing faithful explanations. 

In future work, we plan to investigate the extension of EGAT to other explanation paradigms (e.g., counterfactual-based explanations or attention mechanisms), its applicability to large-scale models and vision-language settings, and the development of more efficient training procedures.

\bibliographystyle{IEEEtran}
\bibliography{main}

@String{Computing = "Computing" }

@String{Computer = "{IEEE} Computer" }

@String{Springer = "Springer-Verlag" }

@inproceedings{guesmi2024asgt,
  title={Exploring the Interplay of Interpretability and Robustness in Deep Neural Networks: A Saliency-Guided Approach},
  author={Guesmi, Amira and Aswani, Nishant Suresh and Shafique, Muhammad},
  booktitle={2024 IEEE International Conference on Image Processing Challenges and Workshops (ICIPCW)},
  pages={4066--4072},
  year={2024},
  organization={IEEE}
}

@inproceedings{li2023dre,
  title={Are Data-driven Explanations Robust against Out-of-distribution Data?},
  author={Li, Tang and Qiao, Fengchun and Ma, Mengmeng and Peng, Xi},
  booktitle={Proceedings of the IEEE/CVF Conference on Computer Vision and Pattern Recognition},
  pages={3821--3831},
  year={2023}
}

@article{weber2024applications,
  title={Applications of explainable artificial intelligence in finance—a systematic review of finance, information systems, and computer science literature},
  author={Weber, Patrick and Carl, K Valerie and Hinz, Oliver},
  journal={Management Review Quarterly},
  volume={74},
  number={2},
  pages={867--907},
  year={2024},
  publisher={Springer}
}

@inproceedings{rieger2020cdep,
  title={Interpretations are useful: penalizing explanations to align neural networks with prior knowledge},
  author={Rieger, Laura and Singh, Chandan and Murdoch, William and Yu, Bin},
  booktitle={International conference on machine learning},
  pages={8116--8126},
  year={2020},
  organization={PMLR}
}

@inproceedings{ross2018improving,
  title={Improving the adversarial robustness and interpretability of deep neural networks by regularizing their input gradients},
  author={Ross, Andrew and Doshi-Velez, Finale},
  booktitle={Proceedings of the AAAI conference on artificial intelligence},
  volume={32},
  number={1},
  year={2018}
}

@article{ross2017right,
  title={Right for the right reasons: Training differentiable models by constraining their explanations},
  author={Ross, Andrew Slavin and Hughes, Michael C and Doshi-Velez, Finale},
  journal={arXiv preprint arXiv:1703.03717},
  year={2017}
}

@article{ismail2021sgt,
  title={Improving deep learning interpretability by saliency guided training},
  author={Ismail, Aya Abdelsalam and Corrada Bravo, Hector and Feizi, Soheil},
  journal={Advances in Neural Information Processing Systems},
  volume={34},
  pages={26726--26739},
  year={2021}
}

@article{bertoin2024sgdrop,
  title={The Overfocusing Bias of Convolutional Neural Networks: A Saliency-Guided Regularization Approach},
  author={Bertoin, David and Sanchez, Eduardo Hugo and Zouitine, Mehdi and Rachelson, Emmanuel},
  journal={arXiv preprint arXiv:2409.17370},
  year={2024}
}

@article{guidedlearning,
  title={Going beyond xai: A systematic survey for explanation-guided learning},
  author={Gao, Yuyang and Gu, Siyi and Jiang, Junji and Hong, Sungsoo Ray and Yu, Dazhou and Zhao, Liang},
  journal={ACM Computing Surveys},
  year={2024},
}

@inproceedings{terra,
  title={Adversarial domain adaptation with domain mixup},
  author={Xu, Minghao and Zhang, Jian and Ni, Bingbing and Li, Teng and Wang, Chengjie and Tian, Qi and Zhang, Wenjun},
  booktitle={AAAI},
  year={2020}
}

@inproceedings{VLCS,
  title={Unbiased metric learning: On the utilization of multiple datasets and web images for softening bias},
  author={Fang, Chen and Xu, Ye and Rockmore, Daniel N},
  booktitle={ICCV},
  year={2013}
}

@inproceedings{he2016resnet,
  title={Deep residual learning for image recognition},
  author={He, Kaiming and Zhang, Xiangyu and Ren, Shaoqing and Sun, Jian},
  booktitle={Proceedings of the IEEE conference on computer vision and pattern recognition},
  pages={770--778},
  year={2016}
}

@inproceedings{selvaraju2019hint,
  title={Taking a hint: Leveraging explanations to make vision and language models more grounded},
  author={Selvaraju, Ramprasaath R and Lee, Stefan and Shen, Yilin and Jin, Hongxia and Ghosh, Shalini and Heck, Larry and Batra, Dhruv and Parikh, Devi},
  booktitle={Proceedings of the IEEE/CVF international conference on computer vision},
  pages={2591--2600},
  year={2019}
}

@article{goodfellow2014fgsm,
  title={Explaining and harnessing adversarial examples},
  author={Goodfellow, Ian J and Shlens, Jonathon and Szegedy, Christian},
  journal={arXiv preprint arXiv:1412.6572},
  year={2014}
}

@article{madry2017pgd,
  title={Towards deep learning models resistant to adversarial attacks},
  author={Madry, Aleksander and Makelov, Aleksandar and Schmidt, Ludwig and Tsipras, Dimitris and Vladu, Adrian},
  journal={arXiv preprint arXiv:1706.06083},
  year={2017}
}

@inproceedings{ghorbani2019interpretation,
  title={Interpretation of neural networks is fragile},
  author={Ghorbani, Amirata and Abid, Abubakar and Zou, James},
  booktitle={Proceedings of the AAAI conference on artificial intelligence},
  volume={33},
  number={01},
  pages={3681--3688},
  year={2019}
}

@article{chen2024training,
  title={Training for stable explanation for free},
  author={Chen, Chao and Guo, Chenghua and Chen, Rufeng and Ma, Guixiang and Zeng, Ming and Liao, Xiangwen and Zhang, Xi and Xie, Sihong},
  journal={Advances in Neural Information Processing Systems},
  volume={37},
  pages={3421--3457},
  year={2024}
}

@article{chen2019ignorm,
  title={Robust attribution regularization},
  author={Chen, Jiefeng and Wu, Xi and Rastogi, Vaibhav and Liang, Yingyu and Jha, Somesh},
  journal={Advances in Neural Information Processing Systems},
  volume={32},
  year={2019}
}

@inproceedings{sarkar2021enhanced,
  title={Enhanced regularizers for attributional robustness},
  author={Sarkar, Anindya and Sarkar, Anirban and Balasubramanian, Vineeth N},
  booktitle={Proceedings of the AAAI Conference on Artificial Intelligence},
  volume={35},
  number={3},
  pages={2532--2540},
  year={2021}
}

@article{wicker2022robust,
  title={Robust explanation constraints for neural networks},
  author={Wicker, Matthew and Heo, Juyeon and Costabello, Luca and Weller, Adrian},
  journal={arXiv preprint arXiv:2212.08507},
  year={2022}
}

@inproceedings{slack2020fooling,
  title={Fooling lime and shap: Adversarial attacks on post hoc explanation methods},
  author={Slack, Dylan and Hilgard, Sophie and Jia, Emily and Singh, Sameer and Lakkaraju, Himabindu},
  booktitle={Proceedings of the AAAI/ACM Conference on AI, Ethics, and Society},
  pages={180--186},
  year={2020}
}

@article{tsipras2018robustness,
  title={Robustness may be at odds with accuracy},
  author={Tsipras, Dimitris and Santurkar, Shibani and Engstrom, Logan and Turner, Alexander and Madry, Aleksander},
  journal={arXiv preprint arXiv:1805.12152},
  year={2018}
}

@article{selvaraju2020gradcam,
  title={Grad-CAM: visual explanations from deep networks via gradient-based localization},
  author={Selvaraju, Ramprasaath R and Cogswell, Michael and Das, Abhishek and Vedantam, Ramakrishna and Parikh, Devi and Batra, Dhruv},
  journal={International journal of computer vision},
  volume={128},
  pages={336--359},
  year={2020},
  publisher={Springer}
}

@inproceedings{zhang2023magi,
  title={Magi: Multi-annotated explanation-guided learning},
  author={Zhang, Yifei and Gu, Siyi and Gao, Yuyang and Pan, Bo and Yang, Xiaofeng and Zhao, Liang},
  booktitle={Proceedings of the IEEE/CVF International Conference on Computer Vision},
  pages={1977--1987},
  year={2023}
}

@inproceedings{zhuang2019care,
  title={Care: Class attention to regions of lesion for classification on imbalanced data},
  author={Zhuang, Jiaxin and Cai, Jiabin and Wang, Ruixuan and Zhang, Jianguo and Zheng, Weishi},
  booktitle={International Conference on Medical Imaging with Deep Learning},
  pages={588--597},
  year={2019},
  organization={PMLR}
}

@inproceedings{guo2019mixup,
  title={Mixup as locally linear out-of-manifold regularization},
  author={Guo, Hongyu and Mao, Yongyi and Zhang, Richong},
  booktitle={Proceedings of the AAAI conference on artificial intelligence},
  volume={33},
  number={01},
  pages={3714--3722},
  year={2019}
}

@inproceedings{xu2020dmada,
  title={Adversarial domain adaptation with domain mixup},
  author={Xu, Minghao and Zhang, Jian and Ni, Bingbing and Li, Teng and Wang, Chengjie and Tian, Qi and Zhang, Wenjun},
  booktitle={Proceedings of the AAAI conference on artificial intelligence},
  volume={34},
  number={04},
  pages={6502--6509},
  year={2020}
}

@article{arjovsky2019irm,
  title={Invariant risk minimization},
  author={Arjovsky, Martin and Bottou, L{\'e}on and Gulrajani, Ishaan and Lopez-Paz, David},
  journal={arXiv preprint arXiv:1907.02893},
  year={2019}
}

@article{viallard2021pac,
  title={A pac-bayes analysis of adversarial robustness},
  author={Viallard, Paul and VIDOT, Eric Guillaume and Habrard, Amaury and Morvant, Emilie},
  journal={Advances in Neural Information Processing Systems},
  volume={34},
  pages={14421--14433},
  year={2021}
}

@article{xiao2022adversarial,
  title={Adversarial rademacher complexity of deep neural networks},
  author={Xiao, Jiancong and Fan, Yanbo and Sun, Ruoyu and Luo, Zhi-Quan},
  journal={arXiv preprint arXiv:2211.14966},
  year={2022}
}

@inproceedings{tan2023robust,
  title={Robust explanation for free or at the cost of faithfulness},
  author={Tan, Zeren and Tian, Yang},
  booktitle={International conference on machine learning},
  pages={33534--33562},
  year={2023},
  organization={PMLR}
}

@article{bartlett2002rademacher,
  title={Rademacher and gaussian complexities: Risk bounds and structural results},
  author={Bartlett, Peter L and Mendelson, Shahar},
  journal={Journal of Machine Learning Research},
  volume={3},
  number={Nov},
  pages={463--482},
  year={2002}
}

@article{bhusaltowards,
  title={Towards improving saliency map interpretability using feature map smoothing},
  author={Bhusal, Dipkamal and Alam, Md Tanvirul and manikya Veerabhadran, Monish Kumar and Clifford, Michael and Rampazzi, Sara and Rastogi, Nidhi},
  year={2025}
}

@inproceedings{dong2018mifgsm,
  title={Boosting adversarial attacks with momentum},
  author={Dong, Yinpeng and Liao, Fangzhou and Pang, Tianyu and Su, Hang and Zhu, Jun and Hu, Xiaolin and Li, Jianguo},
  booktitle={Proceedings of the IEEE conference on computer vision and pattern recognition},
  pages={9185--9193},
  year={2018}
}

@article{shi2024npat,
  title={Unsupervised Class-Imbalanced Domain Adaptation With Pairwise Adversarial Training and Semantic Alignment},
  author={Shi, Weili and Zhu, Ronghang and Li, Sheng},
  journal={IEEE Transactions on Circuits and Systems for Video Technology},
  year={2024},
  publisher={IEEE}
}

@article{wu2025aft,
  title={Adversarial Feature Training for Few-shot Object Detection},
  author={Wu, Tianxu and Xin, Zhimeng and Chen, Shiming and Zou, Yixiong and You, Xinge},
  journal={IEEE Transactions on Circuits and Systems for Video Technology},
  year={2025},
  publisher={IEEE}
}

@article{li2024towards,
  title={Towards Explainable Image Aesthetics Assessment with Attribute-oriented Critiques Generation},
  author={Li, Leida and Sheng, Xiangfei and Chen, Pengfei and Wu, Jinjian and Dong, Weisheng},
  journal={IEEE Transactions on Circuits and Systems for Video Technology},
  year={2024},
  publisher={IEEE}
}

@article{zhao2022toward,
  title={Toward explainable 3D grounded visual question answering: A new benchmark and strong baseline},
  author={Zhao, Lichen and Cai, Daigang and Zhang, Jing and Sheng, Lu and Xu, Dong and Zheng, Rui and Zhao, Yinjie and Wang, Lipeng and Fan, Xibo},
  journal={IEEE Transactions on Circuits and Systems for Video Technology},
  volume={33},
  number={6},
  pages={2935--2949},
  year={2022},
  publisher={IEEE}
}

@article{li2022adal,
  title={Artifacts-disentangled adversarial learning for deepfake detection},
  author={Li, Xin and Ni, Rongrong and Yang, Pengpeng and Fu, Zhiqiang and Zhao, Yao},
  journal={IEEE Transactions on Circuits and Systems for Video Technology},
  volume={33},
  number={4},
  pages={1658--1670},
  year={2022},
  publisher={IEEE}
}

@article{liu2024xfmp,
  title={XFMP: A Benchmark for Explainable Fine-grained Abnormal Behavior Recognition on Medical Personal Protective Equipment},
  author={Liu, Jiaxi and Niu, Jinghao and Li, Weifeng and Li, Xin and He, Binbin and Zhou, Hao and Liu, Yanjuan and Li, Ding and Wang, Bo and Zhang, Wensheng},
  journal={IEEE Transactions on Circuits and Systems for Video Technology},
  year={2024},
  publisher={IEEE}
}

@inproceedings{shen2021iai,
  title={Human-AI interactive and continuous sensemaking: A case study of image classification using scribble attention maps},
  author={Shen, Haifeng and Liao, Kewen and Liao, Zhibin and Doornberg, Job and Qiao, Maoying and Van Den Hengel, Anton and Verjans, Johan W},
  booktitle={extended abstracts of the 2021 CHI conference on human factors in computing systems},
  pages={1--8},
  year={2021}
}

@article{adebayo2018sanitychecks,
  title={Sanity checks for saliency maps},
  author={Adebayo, Julius and Gilmer, Justin and Muelly, Michael and Goodfellow, Ian and Hardt, Moritz and Kim, Been},
  journal={Advances in neural information processing systems},
  volume={31},
  year={2018}
}

@article{Cheng2025ExplainAttack3D,
  title={Black-box explainability-guided adversarial attack for 3D object tracking},
  author={Cheng, Riran and Wang, Xupeng and Sohel, Ferdous and Lei, Hang},
  journal={IEEE transactions on circuits and systems for video technology},
  year={2025},
  publisher={IEEE}
}

@article{Sun2024TargetedAttacksOD,
  title={Task-specific importance-awareness matters: On targeted attacks against object detection},
  author={Sun, Xuxiang and Cheng, Gong and Li, Hongda and Peng, Hongyu and Han, Junwei},
  journal={IEEE Transactions on Circuits and Systems for Video Technology},
  year={2024},
  publisher={IEEE}
}

@article{Fu2024AttackInvariance,
  title={Remove to Regenerate: Boosting Adversarial Generalization With Attack Invariance},
  author={Fu, Xiaowei and Ma, Lina and Zhang, Lei},
  journal={IEEE Transactions on Circuits and Systems for Video Technology},
  year={2024},
  publisher={IEEE}
}

@article{Chen2023RobustNIC,
  title={Toward robust neural image compression: Adversarial attack and model finetuning},
  author={Chen, Tong and Ma, Zhan},
  journal={IEEE Transactions on Circuits and Systems for Video Technology},
  volume={33},
  number={12},
  pages={7842--7856},
  year={2023},
  publisher={IEEE}
}

@article{smilkov2017smoothgrad,
  title={Smoothgrad: removing noise by adding noise},
  author={Smilkov, Daniel and Thorat, Nikhil and Kim, Been and Vi{\'e}gas, Fernanda and Wattenberg, Martin},
  journal={arXiv preprint arXiv:1706.03825},
  year={2017}
}

@inproceedings{sundararajan2017ig,
  title={Axiomatic attribution for deep networks},
  author={Sundararajan, Mukund and Taly, Ankur and Yan, Qiqi},
  booktitle={International conference on machine learning},
  pages={3319--3328},
  year={2017},
  organization={PMLR}
}

@inproceedings{zhang2025finercam,
  title={Finer-cam: Spotting the difference reveals finer details for visual explanation},
  author={Zhang, Ziheng and Gu, Jianyang and Chowdhury, Arpita and Mai, Zheda and Carlyn, David and Berger-Wolf, Tanya and Su, Yu and Chao, Wei-Lun},
  booktitle={Proceedings of the Computer Vision and Pattern Recognition Conference},
  pages={9611--9620},
  year={2025}
}

@inproceedings{sandler2018mobilenetv2,
  title={Mobilenetv2: Inverted residuals and linear bottlenecks},
  author={Sandler, Mark and Howard, Andrew and Zhu, Menglong and Zhmoginov, Andrey and Chen, Liang-Chieh},
  booktitle={Proceedings of the IEEE conference on computer vision and pattern recognition},
  pages={4510--4520},
  year={2018}
}

@inproceedings{tan2019efficientnet,
  title={Efficientnet: Rethinking model scaling for convolutional neural networks},
  author={Tan, Mingxing and Le, Quoc},
  booktitle={International conference on machine learning},
  pages={6105--6114},
  year={2019},
  organization={PMLR}
}

@article{pan2024sca,
  title={SCA: Improve Semantic Consistent in Unrestricted Adversarial Attacks via DDPM Inversion},
  author={Pan, Zihao and Chen, Lifeng and Wu, Weibin and Cao, Yuhang and Zheng, Zibin},
  journal={arXiv preprint arXiv:2410.02240},
  year={2024}
}

@inproceedings{ribeiro2016lime,
  title={" Why should i trust you?" Explaining the predictions of any classifier},
  author={Ribeiro, Marco Tulio and Singh, Sameer and Guestrin, Carlos},
  booktitle={Proceedings of the 22nd ACM SIGKDD international conference on knowledge discovery and data mining},
  pages={1135--1144},
  year={2016}
}

@inproceedings{goyal2019counterfactual,
  title={Counterfactual visual explanations},
  author={Goyal, Yash and Wu, Ziyan and Ernst, Jan and Batra, Dhruv and Parikh, Devi and Lee, Stefan},
  booktitle={International Conference on Machine Learning},
  pages={2376--2384},
  year={2019},
  organization={PMLR}
}

@inproceedings{patro2019u,
  title={U-cam: Visual explanation using uncertainty based class activation maps},
  author={Patro, Badri N and Lunayach, Mayank and Patel, Shivansh and Namboodiri, Vinay P},
  booktitle={Proceedings of the IEEE/CVF International Conference on Computer Vision},
  pages={7444--7453},
  year={2019}
}

\end{document}